\pdfoutput=1

\documentclass[11pt]{article}

\usepackage[final]{acl}

\usepackage{times}
\usepackage{latexsym}

\usepackage[T1]{fontenc}

\usepackage[utf8]{inputenc}

\usepackage{microtype}

\usepackage{inconsolata}

\usepackage{graphicx}

\usepackage{amsmath}
\usepackage{amsthm}
\usepackage{booktabs}
\usepackage{multirow}
\usepackage{caption}
\usepackage{amssymb}
\newtheorem{theorem}{Theorem}
\DeclareMathOperator*{\argmax}{arg\,max}

\title{Diversity Explains Inference Scaling Laws: \\Through a Case Study of Minimum Bayes Risk Decoding}

\author{Hidetaka Kamigaito$^{1}$\textnormal{,}~Hiroyuki Deguchi$^{1}$\textnormal{,}~Yusuke Sakai$^{1}$\textnormal{,}\\\textbf{Katsuhiko Hayashi}$^{2}$, \textbf{Taro Watanabe}$^1$\\$^1$Nara Institute of Science and Technology (NAIST), $^2$The University of Tokyo\\\texttt{\{kamigaito.h, deguchi.hiroyuki.db0, sakai.yusuke.sr9,taro\}@is.naist.jp}\\ \texttt{katsuhiko-hayashi@g.ecc.u-tokyo.ac.jp}}

\begin{document}
\maketitle
\begin{abstract}
Inference methods play an important role in eliciting the performance of large language models (LLMs). Currently, LLMs use inference methods utilizing generated multiple samples, which can be derived from Minimum Bayes Risk (MBR) Decoding. Previous studies have conducted empirical analyses to clarify the improvements in generation performance achieved by MBR decoding and have reported various observations. However, the theoretical underpinnings of these findings remain uncertain. To address this, we offer a new theoretical interpretation of MBR decoding from the perspective of bias-diversity decomposition. In this interpretation, the error in the quality estimation of hypotheses by MBR decoding is decomposed into two main factors: bias, which considers the closeness between the utility function and human evaluation, and diversity, which represents the variability in the quality estimation of the utility function. The theoretical analysis reveals the difficulty of simultaneously improving bias and diversity, confirming the validity of enhancing MBR decoding performance by increasing diversity. Furthermore, we reveal that diversity can explain one aspect of inference scaling laws that describe performance improvement by increasing sample size. Moreover, experiments across multiple NLP tasks yielded results consistent with these theoretical characteristics. Our code is available at \url{https://github.com/naist-nlp/mbr-bias-diversity}.
\end{abstract}

\section{Introduction}

As demonstrated by the success of large language models (LLMs) \citep{NEURIPS2020_1457c0d6,openai2024gpt4technicalreport}, text generation is one of the most fundamental tasks in Natural Language Processing (NLP). In the generation, inference methods play an important role in eliciting model ability. In parallel to the advance of LLMs, inference methods utilizing multiple samples, such as self-consistency (SC) \cite{wang2023selfconsistency} and Complex SC \cite{fu2023complexitybased}, are introduced. 
\citet{bertsch-etal-2023-mbr} prove that these methods can be derived from Minimum Bayes Risk (MBR) decoding \citep{GOEL2000115}.

MBR decoding can elicit models' generation performance by using a utility function, essentially an automatic evaluation metric, along with pseudo-references generated by the model. MBR decoding was initially applied to speech recognition \citep{GOEL2000115} and later to statistical machine translation (SMT) \citep{kumar-byrne-2002-minimum,kumar-byrne-2004-minimum,duan-etal-2011-improving}. Following these successes, MBR decoding has been expanded to various text generation tasks, including neural machine translation (NMT) \citep{stahlberg-etal-2017-neural}, text summarization \citep{bertsch-etal-2023-mbr}, and image captioning \citep{borgeaud-emerson-2020-leveraging}.

Since MBR decoding has become an important inference technique in text generation, various empirical studies have explored its characteristics. \citet{muller-sennrich-2021-understanding,10.1162/tacl_a_00491,fernandes-etal-2022-quality,amrhein-sennrich-2022-identifying} highlight the importance of using high-quality evaluation metrics that are robust and correlate well with human evaluations as utility functions. \citet{jinnai-etal-2024-generating,heineman2024improvingminimumbayesrisk} emphasize the importance of high-quality pseudo-references that closely resemble human-created ones while stressing the significance of pseudo-reference diversity.
Although these empirical findings cover various aspects in detail, a unified interpretation remains challenging due to the lack of theoretical frameworks explaining the relationships behind them.

To address this gap, we provide theoretical interpretations of MBR decoding through bias-diversity decomposition \citep{NIPS1994_b8c37e33,10.5555/3648699.3649058}. Our theoretical interpretation focuses on errors in the estimated quality of hypotheses in MBR decoding. These errors are decomposed into two critical factors: \emph{bias} and \emph{diversity}. The bias term represents the closeness between the estimated quality produced by utility functions and human evaluations. The diversity term reflects the variance in the estimated quality across different utility functions.
Based on this interpretation, we theoretically demonstrate the difficulty of improving both the bias and diversity terms simultaneously and highlight the effectiveness of increasing diversity in MBR decoding, verifying the correspondence with empirically induced results from previous work.

Furthermore, by focusing on information-theoretic diversity \cite{brown2009,zhou2010}, we broaden the scope of our analysis beyond MBR decoding and reveal that diversity is also a key to explaining inference scaling laws \cite{wu2024scaling,brown2024largelanguagemonkeysscaling,chen2025simpleprovablescalinglaws,snell2025scaling} which describe the performance improvement by increasing sample size.

Our empirical analysis on machine translation, text summarization, and image captioning using five different sampling methods shows consistent results with our theoretical analysis.

\section{Minimum Bayes Risk Decoding}

MBR decoding \cite{eikema-aziz-2020-map,eikema-aziz-2022-sampling} estimates the quality of a hypothesis $h$ in the candidate set $\mathcal{H}$ by using a set $\mathcal{Y}$ of pseudo-references $y$ sampled from the model's predicted probability $P_{\pi}(y|x)$ with its model weight $\pi$ for the input sequence $x$. By treating the evaluation metric as a utility function $f_{\mathbf{\theta}}(h, y)$ that measures the similarity between $h$ and $y$, MBR decoding selects the best hypothesis $\hat{h}_{mbr}$ in \(\mathcal{H}\) as:
\begin{align}
   \!\!\hat{h}_{mbr} \!=& \underset{h \in \mathcal{H}}{\operatorname{argmax}} \frac{1}{|\mathcal{Y}|}\!\sum_{y \in \mathcal{Y}} f_{\mathbf{\theta}}(h,y), y \!\sim\! P_{\pi}(y|x).\label{eq:mbr}\\
   \approx & \underset{h \in \mathcal{H}}{\operatorname{argmax}}\sum_{y}f_{\mathbf{\theta}}(h,y)P_{\pi}(y|x)\label{eq:mbr:exp}
\end{align}
Here, $\mathbf{\theta}$ represents the parameters of the evaluation metric used in the utility function $f_{\mathbf{\theta}}(h, y)$.

Alternatively, instead of using the utility function $f_{\mathbf{\theta}}(h, y)$, one can assume the quality estimated by humans \cite{naskar-etal-2023-quality,suzgun-etal-2023-follow,jinnai-etal-2024-generating,ohashi-etal-2024-true} as $\hat{f}_{\hat{\mathbf{\theta}}}(h)$. Under this assumption, the ideal decoding as estimated by humans is given by:
\begin{equation}
   \hat{h}_{human} = \underset{h \in \mathcal{H}}{\operatorname{argmax}} \hat{f}_{\hat{\mathbf{\theta}}}(h).
   \label{eq:human}
\end{equation}
In this paper, we focus on analyzing the differences between the quality estimated by MBR decoding and that estimated by humans to better understand MBR decoding.

\section{Theoretical Analysis based on Bias-Diversity Decomposition}
\label{sec:theoretical}

\subsection{Evaluation Discrepancy}

To measure the discrepancy between the human estimated quality, $\hat{f}_{\hat{\mathbf{\theta}}}(h)$ and the MBR decoding estimated quality, $\frac{1}{|\mathcal{Y}|}\sum_{y \in \mathcal{Y}}f_{\mathbf{\theta}}(h,y)$, we define a $|\mathcal{H}|$-dimensional vector $\textbf{u}^{j}$ that represents estimated quality for each hypothesis based on the $j$-th pseudo-reference and also define $\bar{\textbf{u}}$, the average vector of all $\textbf{u}^{j}$ as follows:
\begin{equation}
   \!\!\textbf{u}^{j} \!=\! \begin{bmatrix}u^{j}_{1}\\\cdots\\u^{j}_{|\mathcal{H}|}\end{bmatrix}, u^{j}_{i} \!=\! f_{\mathbf{\theta}}(h_i,y_j), \bar{\textbf{u}} \!=\! \frac{1}{|\mathcal{Y}|}\sum_{j=1}^{|\mathcal{Y}|}\mathbf{u}^{j}.
   \label{eq:vec:mbr}
\end{equation}
Similarly, we can define a $|\mathcal{H}|$-dimensional vector, $\hat{\textbf{u}}$ that represents the human estimated quality for each hypothesis as follows:
\begin{equation}
   \hat{\textbf{u}} = \begin{bmatrix}\hat{u}_{1}\\\cdots\\\hat{u}_{|\mathcal{H}|}\end{bmatrix},\quad \hat{u}_{i} = \hat{f}_{\hat{\mathbf{\theta}}}(h_i).
   \label{eq:vec:human}
\end{equation}
Here, by using Eqs.~(\ref{eq:vec:mbr}) and (\ref{eq:vec:human}), we can reformulate MBR decoding in Eq.~(\ref{eq:mbr}) and the ideal decoding in Eq.~(\ref{eq:human}) as follows: 
\begin{align}
   (\ref{eq:mbr}) &\equiv \hat{h}_{mbr} = \argmax_{h_{i}}\bar{u}_{i},\nonumber \\
   (\ref{eq:human}) &\equiv \hat{h}_{human} = \argmax_{h_{i}}\hat{u}_{i}.
   \label{eq:reform:dec}
\end{align}
Therefore, based on Eq.~(\ref{eq:reform:dec}), we can investigate the discrepancy between the estimated quality by MBR decoding and human through the comparison of $\bar{\textbf{u}}$ and $\hat{\textbf{u}}$.
In our work, to estimate the discrepancy, we consider the prediction error of $\bar{\textbf{u}}$ to $\hat{\textbf{u}}$ by using Mean Squared Error (MSE) as follows:
\begin{align}
    MSE(\hat{\textbf{u}}, \bar{\textbf{u}}) &= \frac{1}{|\mathcal{H}|}\sum_{i=1}^{|\mathcal{H}|}(\hat{u}_{i} - \bar{u}_{i})^2 \\
    &= \mathbb{E}_{i \in \mathcal{H}}[(\hat{u}_{i} - \bar{u}_{i})^2].\label{eq:mse}
\end{align}

\subsection{Bias-diversity Decomposition}
\label{subsec:decomposition}

Our goal is to reveal the characteristics of MBR decoding through theoretical analysis. To achieve this, we focus on the bias and diversity underlying Eq.~(\ref{eq:mse}). Based on this approach, we can induce the following decomposition:

\begin{theorem}[\textbf{Bias and Diversity Decomposition}]
The quality estimation error for MBR decoding, $MSE(\hat{\textbf{u}}, \bar{\textbf{u}})$, can be decomposed into bias and diversity (ambiguity) terms \cite{NIPS1994_b8c37e33} as follows:
\begin{align}
    &MSE(\hat{\textbf{u}}, \bar{\textbf{u}}) \nonumber\\
    = &\underbrace{\mathbb{E}_{i \in \mathcal{H}}\mathbb{E}_{j \in \mathcal{Y}}[(\hat{u}_{i} - f_{\mathbf{\theta}}(h_i,y_j))^2]}_{\text{Bias}} \nonumber\\
    &- \underbrace{\mathbb{E}_{i \in \mathcal{H}}\mathbb{E}_{j \in \mathcal{Y}}[(\bar{u}_{i} - f_{\mathbf{\theta}}(h_i,y_j))^2]}_{\text{Diversity}}.
    \label{eq:bias-variance}
\end{align}
\label{theorem:bias-variance}
(See Appendix \ref{appendix:proof-bias-variance} for the proof.)
\end{theorem}

In Eq.~(\ref{eq:bias-variance}), two terms represent bias and diversity. Unlike the well-known bias-variance decomposition~\citep{geman1992neural} that targets a single estimator, which is $\mathbf{u}$ in our case, the second term is negative, which is why it is referred to as diversity rather than variance \citep{10.5555/3648699.3649058}. \emph{Bias} indicates how closely the utility function's estimated quality for a hypothesis matches human estimation. \emph{Diversity} reflects how different the utility function's estimated qualities are from each other. This decomposition emphasizes the importance of increasing \emph{Diversity} while reducing \emph{Bias} to improve the quality estimation error, $MSE(\hat{\textbf{u}}, \bar{\textbf{u}})$.

Even though Theorem \ref{theorem:bias-variance} indicates the potential of \emph{Diversity} to improve the performance in MBR decoding, there are some limitations.
\begin{theorem}[\textbf{Limitation of Diversity}]
The decomposition of the quality estimation error for MBR decoding (Eq.~(\ref{eq:bias-variance})) holds the following relation:
\begin{equation}
    Diversity \xrightarrow[Bias\to0]{}0
\end{equation}
\label{theorem:bias-variance-tradeoff-1}
(See Appendix \ref{appendix:proof-bias-variance-tradeoff-1} for the proof.)
\end{theorem}
According to Theorem \ref{theorem:bias-variance-tradeoff-1}, we cannot expect performance gains from \emph{Diversity} when \emph{Bias} is close to zero. Note that this relationship does not guarantee that \emph{Diversity} becomes large when \emph{Bias} is far from zero. As a broadly generalized relationship of this part, the following theorem holds.
\begin{theorem}[\textbf{Bias and Diversity Trade-off}] Bias and Diversity in the decomposition of the quality estimation error for MBR decoding (Eq.~(\ref{eq:bias-variance})) depend on each other based on the following reformulation by \citet{JMLR:v6:brown05a}:
\begin{align}
    Bias &\!=\! \overline{bias}^2 \!+\! \Omega \nonumber\\
    \!\!\!\!Diversity &\!=\! \Omega \!-\! \left[\frac{1}{|\mathcal{Y}|}\overline{var}\!+\!\left(\!1\!-\!\frac{1}{|\mathcal{Y}|}\right)\overline{cov}\right]\!,
    \label{eq:bias_diversity_dependence}
\end{align}
where $\overline{bias}\!\!\!\!=\!\!\!\!\mathbb{E}_{i \in \mathcal{Y}}[\mathbb{E}_{j \in \mathcal{H}}[u_j^i] \!\!\!-\!\!\! \mathbb{E}_{j \in \mathcal{H}}[\hat{u}_j]], \overline{var}=$ $\mathbb{E}_{i \in \mathcal{Y}}[\mathbb{E}_{j \in \mathcal{H}}[(u_j^i - \mathbb{E}_{k \in \mathcal{H}}[u_k^i])^2]]$, $\overline{cov}=\frac{1}{|\mathcal{Y}|(|\mathcal{Y}|-1)}$ $\sum_{i=1}^{|\mathcal{Y}|}\!\!\sum_{i \not = j}^{|\mathcal{Y}|}\mathbb{E}_{k \in \mathcal{H}}[(u_k^i \!\!-\!\! \mathbb{E}_{l \in \mathcal{H}}[u_l^i])(u_k^i \!\!-\!\! \mathbb{E}_{l \in \mathcal{H}}[u_l^j])]$, and $\Omega = \overline{var} + \mathbb{E}_{i \in \mathcal{Y}}[(\mathbb{E}_{k \in \mathcal{H}}[u_k^i] - \mathbb{E}_{k \in \mathcal{H}}[\overline{u}_k])^2]$.
\label{theorem:bias-variance-tradeoff-2}
(See Appendix \ref{appendix:proof-bias-diversity-dependence} for the proof.)
\end{theorem}
In Eq.~(\ref{eq:bias_diversity_dependence}), \emph{Bias} and \emph{Diversity} share $\Omega$. Thus, Theorem \ref{theorem:bias-variance-tradeoff-2} demonstrates the general trade-off relationship between \emph{Bias} and \emph{Diversity}, underscoring the difficulty of maximizing \emph{Diversity} without affecting \emph{Bias}.

\subsection{Diversity behind Inference Scaling Laws}

In some applications, MBR decoding is part of a reranking system that employs a scoring function other than $f_{\mathbf{\theta}}(h,y)$. In addition, inference methods are not restricted to the one derived from MBR decoding. To broaden our scope of analysis beyond MBR decoding and analyze the inference scaling laws, we rely on the following theorem.
\begin{theorem}[\textbf{Generalized Diveristy}]
Letting $\hat{\mathcal{H}}$ be a distribution for the human-selected candidate, $\mathcal{X}_{1:|\mathcal{Y}|}$ be representations corresponding to all elements in $\mathcal{Y}$, and $g(\mathcal{X}_{1:|\mathcal{Y}|})$ be a function predicting the best candidate. Its prediction error $p(\hat{\mathcal{H}} \not = g(\mathcal{X}_{1:|\mathcal{Y}|}))$ satisfies the following inequality \cite{brown2009,zhou2010}:
\begin{align}
   &\frac{H(\hat{\mathcal{H}})-I(\mathcal{X}_{1:|\mathcal{Y}|};\hat{\mathcal{H}})-1}{\log{|\hat{\mathcal{H}}|}} \leq p(\hat{\mathcal{H}} \not = g(\mathcal{X}_{1:|\mathcal{Y}|})), \nonumber\\
   &p(\hat{\mathcal{H}} \not = g(\mathcal{X}_{1:|\mathcal{Y}|})) \leq \frac{H(\hat{\mathcal{H}}) \!-\!I(\mathcal{X}_{1:|\mathcal{Y}|};\hat{\mathcal{H}})}{2},\label{eq:bounds}
\end{align}
where $H$ is entropy and $I$ is mutual information. To minimize the error, we should maximize $I(\mathcal{X}_{1:|\mathcal{Y}|};\hat{\mathcal{H}})$,  decomposed to \cite{zhou2010}:
\begin{equation}
\!\!\!\!\!\!\underbrace{\sum_{i=1}^{|\mathcal{Y}|}I(\mathcal{X}_{i};\hat{\mathcal{H}})}_{Relevancy}+\!\!\!\!\!\underbrace{\underbrace{\mathcal{I}(\mathcal{X}_{1:|\mathcal{Y}|}|\hat{\mathcal{H}})}_{Conditional\ Redundancy} \!\!\!-\!\!\underbrace{\mathcal{I}(\mathcal{X}_{1:|\mathcal{Y}|})}_{Redundancy}}_{Information-Theoretic\ Diversity},
\label{eq:it_diversity}
\end{equation}
where $\mathcal{I}(\mathcal{X}_{1:|\mathcal{Y}|})$ and $\mathcal{I}(\mathcal{X}_{1:|\mathcal{Y}|}|\hat{\mathcal{H}})$ are total correlation and conditional total correlation, respectively.
\label{theorem:bias-variance-general}
(See Appendix \ref{appendix:proof-it-diversity} for the proof.)
\end{theorem}
For maximizing \emph{Information Theoretic Diveristy} in Eq.~(\ref{eq:it_diversity}), \emph{Redundancy} must be zero. This condition is satisfied when all elements in $\mathcal{X}_{1:|\mathcal{Y}|}$ are independent of each other. That shows the importance of the diversity of generated samples in $\mathcal{Y}$. This theoretical aspect is consistent with the importance of the diversity of  $\mathcal{Y}$ shown in Eq.~(\ref{eq:bias-variance}). 
\begin{theorem}[\textbf{Monotonicity of Terms}]
When $\mathcal{X}_{1:|\mathcal{Y}|}$ increases (that means new elements are added to $\mathcal{X}_{1:|\mathcal{Y}|}$), \textit{Relevancy}, \textit{Conditional Redundancy}, and \textit{Redundancy} monotonically increase, while \textit{Information Theoretic Diveristy} and $I(\mathcal{X}_{1:|\mathcal{Y}|};\hat{\mathcal{H}})$ do not monotonically increase or decrease. (See Appendix \ref{appendix:proof-monotonicity-terms} for the proof.)
\label{theorem:monotonicity-terms}
\end{theorem}
This theorem concludes that just increasing sample size does not guarantee performance improvement in inference due to the degradation of \textit{Information Theoretic Diversity}. Thus, diversity is still a key to performance improvement, also from the information-theoretic interpretation. The bounds further support this characteristic as follows.
\begin{theorem}[\textbf{Monotonicity of Bounds}]
The lower and upper bounds of the error in Eq.~(\ref{eq:bounds}) do not monotonically increase or decrease when $\mathcal{X}_{1:|\mathcal{Y}|}$ increases. (See Appendix \ref{appendix:monotonicity-bounds} for the proof.)
\label{theorem:monotonicity-bounds}
\end{theorem}
However, these characteristics contradict previous work reporting performance improvement by increasing sample size. The following theorem can fill in the gap.
\begin{theorem}[\textbf{Submodularity of Terms}]
When variables in $\mathcal{X}_{1:|\mathcal{Y}|}$ are independent given $\hat{\mathcal{H}}$, $I(\mathcal{X}_{1:|\mathcal{Y}|};\hat{\mathcal{H}})$ and Information-Theoretic Diversity have submodularity on $\mathcal{X}_{1:|\mathcal{Y}|}$ and $I(\mathcal{X}_{1:|\mathcal{Y}|};\hat{\mathcal{H}})$ does not decrease when $\mathcal{X}_{1:|\mathcal{Y}|}$ increases. (See Appendix \ref{appendix:submodularity:terms} for the proof.)
\label{theorem:submodularity:terms}
\end{theorem}
This theorem shows a kind of lower bound for the performance because \textit{Conditional Redundancy} becomes zero under the condition.
Here, submodularity is a characteristic that the effect for performance improvement by increasing $\mathcal{X}_{1:|\mathcal{Y}|}$ decreases.
Thus, under the conditional independence of $\mathcal{X}_{1:|\mathcal{Y}|}$ given $\hat{\mathcal{H}}$, increasing $\mathcal{X}_{1:|\mathcal{Y}|}$ makes $I(\mathcal{X}_{1:|\mathcal{Y}|};\hat{\mathcal{H}})$ increase and gradually converge based on its submodularity. This behavior is along with the inference scaling laws that indicate the logarithmic convergence of performance by increasing sample size. The following theorem supports this aspect.
\begin{theorem}[\textbf{Supermodularity of Bounds}] When variables in $\mathcal{X}_{1:|\mathcal{Y}|}$ are independent given $\hat{\mathcal{H}}$, the lower and upper bounds of the error in Eq.~(\ref{eq:bounds}) are supermodular and non-increasing on $\mathcal{X}_{1:|\mathcal{Y}|}$. (See Appendix \ref{appendix:submodularity:bounds} for the proof.)
\label{theorem:submodularity:bounds}
\end{theorem}
Therefore, prediction error $p(\hat{\mathcal{H}} \not = g(\mathcal{X}_{1:|\mathcal{Y}|}))$ decreases and converges corresponding to the increase of $I(\mathcal{X}_{1:|\mathcal{Y}|};\hat{\mathcal{H}})$, since supermodularity is the negative of submodularity. This result supports the correspondence with inference scaling laws.

\subsection{Interpretation}

The decompositions and their details presented in these theorems allow us to provide theoretical interpretations for the empirically analyzed characteristics of MBR decoding and other inference methods in prior studies.

\subsubsection{Correlation to Human Evaluation Results}
\label{subsubsec:corr_human}

\emph{Bias} in Theorem~\ref{theorem:bias-variance} highlights the importance of considering the closeness between the human-estimated quality, $\hat{u}_{i}$, and the quality estimated by the utility function, $f_{\mathbf{\theta}}(h_i,y_j)$, for improving the performance of MBR decoding. Specifically, since the utility function, $f_{\mathbf{\theta}}(h_i,y_j)$, is influenced by the pseudo-reference $y_j$, \emph{Bias} underscores the significance of considering the utility function's correlation to human evaluation and the closeness between pseudo-references and human-created references. Therefore, it emphasizes the importance of examining both utility functions and sampling strategies for generating pseudo-references.

\noindent\textbf{Quality of Evaluation Metrics.} \citet{muller-sennrich-2021-understanding,10.1162/tacl_a_00491,fernandes-etal-2022-quality,amrhein-sennrich-2022-identifying} support our theoretical insight, in which the quality of evaluation metrics used as utility functions is crucial for performance improvement.

\noindent\textbf{Quality of Pseudo-References.} \citet{ohashi-etal-2024-true,jinnai-etal-2024-generating} empirically show the importance of selecting appropriate pseudo-references. Our findings theoretically support these empirical insights.

\noindent\textbf{Challenges in the Real World.} Our theoretical findings emphasize the necessity of directly reducing the bias term. However, this requires human evaluation of the combination of pseudo-references and evaluation metrics, used as utility functions, for each hypothesis. This task is clearly challenging due to the high cost of human evaluation. As a solution, we suggest a method to approximate this in \S\ref{subsec:pseudo_bias} and evaluate its correlation with task-specific performance in \S\ref{subsec:experiments:bias-diversity}.

\subsubsection{Diversity of Automatic Evaluation Results}
\label{subsubsec:diversity}

\emph{Diversity} in Theorem~\ref{theorem:bias-variance} shows that increasing diversity can contribute to performance improvements in MBR decoding. A key insight here is that the diversity expressed by $(\bar{u}_{i} - f_{\mathbf{\theta}}(h_i,y_j))^2$ stems from the different estimated qualities produced by each utility function $f_{\mathbf{\theta}}(h_i,y_j)$. Thus, this diversity can be influenced by the pseudo-reference $y_j$ and/or the model parameters $\mathbf{\theta}$ of the evaluation metric.

\noindent\textbf{Diversity of Pseudo-references.} This finding supports the previous studies \citep{freitag-etal-2023-epsilon,jinnai-etal-2024-generating,heineman2024improvingminimumbayesrisk} that conclude the diversity of sampling methods is essential for performance improvement of MBR decoding considering that the diversity of the pseudo-references can indirectly contribute to increasing the diversity of $f_{\mathbf{\theta}}(h_i,y_j)$ by each $y_j$.

\noindent\textbf{Diversity of Evaluation Metrics.} Theoretically, we can anticipate performance improvements by combining multiple different evaluation metrics as utility functions to increase diversity. This is supported by empirical insights from \citet{kovacs-etal-2024-mitigating} in MBR decoding and \citet{glushkova-etal-2023-bleu} in the quality estimation in text generation.

\noindent\textbf{Unexplored Aspect.} Furthermore, the effect of increasing the diversity of estimated qualities from utility functions by varying the evaluation metric’s model parameters $\mathbf{\theta}$ remains uncertain. To investigate this, we propose a method to adjust the diversity of estimated qualities by modifying $\mathbf{\theta}$ in \S\ref{subsec:mambr} and compare its behavior with that of varying pseudo-references in \S\ref{subsec:experiments:mambr}.

\subsubsection{MBR Decoding as Ensemble Learning}
\label{subsubsec:ensemble}

Our decomposition in Theorem~\ref{theorem:bias-variance} aligns with ensemble learning, which is induced by \citet{NIPS1994_b8c37e33}. Thus, we can understand that the quality estimation by MBR decoding is a kind of ensemble learning.

\noindent\textbf{Quality Estimation.} Our decomposition starts from the definition $MSE(\hat{\textbf{u}}, \bar{\textbf{u}})$ in Eq.~(\ref{eq:mse}), the error between the estimated qualities from human evaluation and MBR decoding. We can actually observe the reduction of errors as the improvement in quality score estimation of \citep{naskar-etal-2023-quality,cheng-vlachos-2024-measuring} by ensembling utility functions that are similar to MBR decoding.

\noindent\textbf{Weighted-voting.} Furthermore, this viewpoint supports the validity of the previous work \citep{suzgun-etal-2023-follow,bertsch-etal-2023-mbr} that shows the interpretation of MBR decoding as soft-weighted voting, a variant of ensemble learning. Different from us, soft-weighted voting restricts the value range of voters (utility functions) from 0 to 1. \citet{10.5555/3648699.3649058} shows that soft-weighted voting can be converted to the decomposition of \citet{NIPS1994_b8c37e33}, equivalent to our decomposition in Eq.~(\ref{eq:bias-variance}). Therefore, weighted voting-based MBR decoding can be similarly explained in our decomposition. 

\noindent\textbf{Number of Pseudo-references.} Generally, increasing the number of pseudo-references improves performance but demands additional computational costs. \citet{denero-etal-2009-fast,eikema-aziz-2022-sampling,cheng-vlachos-2023-faster,deguchi-etal-2024-centroid,vamvas-sennrich-2024-linear,trabelsi2024efficientminimumbayesrisk} prune samples to speed up inference and maintain the original quality similar to the case of pruning estimators in ensemble learning \citep{liu2004empirical,hamed2016cikm,hamed2019tnnls}.

Considering an ensemble learning method, such as the Bayes optimal classifier \citep{mitchell1997introduction}, and assuming that Eq.~(\ref{eq:mbr}) approximates the expectation by sampling $y_j$, we can explain the performance improvement of increased pseudo-references by the law of large numbers and the success of the pruning and weighted utility functions \citep{pmlr-v235-jinnai24a} through importance sampling \citep{kloek1978bayesian}. (See Appendix \ref{appendix:interpretation:ensemble-learning} for more details.)

\subsubsection{Bias and Diversity Trade-off}

At first glance, based on the interpretation in \S\ref{subsubsec:corr_human} and \S\ref{subsubsec:diversity}, decreasing \emph{Bias} while increasing \emph{Diversity} seems to be the best strategy to improve performance in MBR decoding, which was investigated by \citet{jinnai-etal-2024-generating}. To understand its validity, we need to focus on Theorems~\ref{theorem:bias-variance-tradeoff-1} and \ref{theorem:bias-variance-tradeoff-2}.

\noindent\textbf{Limitation of MBR Decoding.} Theorem \ref{theorem:bias-variance-tradeoff-1} highlights the difficulty of increasing \emph{Diversity} when \emph{Bias} is close to zero. This theoretical fact indicates that even if we can prepare high-quality evaluation metrics and pseudo-references that correlate well with human behavior, there may be no performance improvement due to diminished \emph{Diversity}. Furthermore, Theorem \ref{theorem:bias-variance-tradeoff-2} highlights that even when \emph{Bias} is not close to zero, \emph{Bias} influences \emph{Diversity}.

\noindent\textbf{Diversity Assists Inferior Methods.} Conversely, when the evaluation metrics and pseudo-references are inferior, we can expect performance improvements through \emph{Diversity} at the cost of increased \emph{Bias}. This phenomenon can explain the sometimes competitive performance of BLEU \citep{papineni-etal-2002-bleu} against COMET \citep{rei-etal-2020-comet} in \citet{freitag-etal-2022-results}, and that of ancestral sampling \citep{robert1999monte} against other sampling methods in \citet{freitag-etal-2023-epsilon,ohashi-etal-2024-true} using MBR decoding. However, increased \emph{Bias} does not guarantee increased \emph{Diversity}. Therefore, we must carefully assess their diversity when using low-quality evaluation metrics and pseudo-references.

\subsubsection{Diversity of Hypotheses}
\label{subsec:diversity:hypo}
By replacing $\mathcal{X}_{1:|\mathcal{Y}|}$ in Theorem \ref{theorem:bias-variance-general} with $\mathcal{X}_{1:|\mathcal{Y}\cup \mathcal{H}|}$, representations corresponding to all elements in $\mathcal{Y}\cup \mathcal{H}$, we can show the importance of diversity and sample size for hypotheses, through Theorems \ref{theorem:bias-variance-general}, \ref{theorem:monotonicity-terms}, and \ref{theorem:monotonicity-bounds}, theoretically. That corresponds to the previous empirical studies, which focus on the sample size of hypotheses \citep{eikema-aziz-2020-map, fernandes-etal-2022-quality, freitag-etal-2023-epsilon}.

\subsubsection{Diversity in Various Decoding Methods}
\label{subsec:diversity:decoding}
Theorems \ref{theorem:bias-variance-general}, \ref{theorem:monotonicity-terms}, and \ref{theorem:monotonicity-bounds} demonstrate the importance of diversity for various decoding methods in addition to MBR decoding and its variants, such as universal self-consistency \cite{chen2024universal}, reranking by rewards \cite{wu2024scaling}, and the combination with other reranking methods \cite{kovacs-etal-2024-mitigating,lyu2025unveilingpowersourcesourcebased}. Basically, these theorems are applicable to other various decoding methods when they use hypotheses supported by \S\ref{subsec:diversity:hypo}.

\subsubsection{Inference Scaling Laws}

Theorems \ref{theorem:submodularity:terms} and \ref{theorem:submodularity:bounds} show the theoretical background of inference scaling laws \cite{wu2024scaling,brown2024largelanguagemonkeysscaling,chen2025simpleprovablescalinglaws,snell2025scaling}, which are applicable to various decoding methods supported by \S\ref{subsec:diversity:hypo} and \S\ref{subsec:diversity:decoding}. This theoretical finding is based on the conditional independence of the representations of the samples given a correct output. This condition is natural in many inference approaches, where samples are independently generated to cover the correct output. Furthermore, since the characteristic shown in Theorems \ref{theorem:submodularity:terms} and \ref{theorem:submodularity:bounds} is a kind of lower bound, we can expect further performance gains by improved \textit{Conditional Redundancy}, interaction of samples to answer. In this situation, the performance improvement shows non-monotonic behavior. We check that in \S\ref{subsec:inference_scaling_laws}.

\section{Remaining Problems \& Solutions}

Our theoretical analysis covers various aspects of MBR decoding. However, for a comprehensive analysis, we should investigate empirical results not addressed in previous work and bridge the gap between theory and real-world applications. To this end, we provide the following solutions.

\subsection{Pseudo-Bias}
\label{subsec:pseudo_bias}

As discussed in \S\ref{subsubsec:corr_human}, the bias term suggests the importance of considering the correlation between the results of human evaluation and the evaluation metric's decisions based on pseudo-references to improve the performance of MBR decoding. However, calculating the bias term requires human evaluation, and conducting human evaluations for each setting is unrealistic and difficult. To address this issue, we introduce \textit{pseudo-bias}, an approximation of the bias term in our decomposition. By using $|\hat{\mathcal{Y}}|$, the number of gold references $\hat{y}$, pseudo-bias is defined as follows:
\begin{equation}
\frac{1}{|\mathcal{H}|}\sum_{i=1}^{|\mathcal{H}|}\frac{1}{|\mathcal{Y}|}\sum_{j=1}^{|\mathcal{Y}|}(\widetilde{u}_{i} - u_{i}^{j})^2,
\end{equation}
where $\widetilde{u}_{i} = \frac{1}{|\hat{\mathcal{Y}}|}\sum_{j=1}^{|\hat{\mathcal{Y}}|}f_{\mathbf{\theta}}(h_i,\hat{y}_{j})$. This formulation is based on the premise that automatic evaluation metrics correlate to human evaluation when receiving human-created references.\footnote{For the pseudo-bias, we used COMET (\texttt{Unbabel/wmt22-comet-da}) and BERTScore with \texttt{microsoft/deberta-xlarge-mnli} whose pearson correlations are 0.990 on the system-level task for English to German \citep{freitag-etal-2023-results} and 0.7781 (\url{https://github.com/Tiiiger/bert_score}) on WMT16 to English \citep{bojar-etal-2016-results}, respectively.}
Since we can calculate the diversity term without any approximation, we compare pseudo-bias with diversity in terms of how they correlate with performance.
\subsection{Metric-augmented MBR}
\label{subsec:mambr}

The discussion in \S\ref{subsubsec:diversity} shows the possibility of increasing the diversity of the utility function, $f_{\theta}(h_i, y_{j})$, by changing the evaluation metric's model parameters, $\theta$, as well as by introducing diversity through pseudo-references. To this end, we propose a new method called Metric-augmented Minimum Bayes Risk (MAMBR) decoding. In MAMBR, we employ different parameters for the evaluation metric to enhance the diversity of utility functions. Letting $\Theta$ be a set of model parameters, MAMBR is defined as follows:
\begin{equation}
   \!\!\!\hat{h}_{\text{mambr}} = \arg\max_{h \in \mathcal{H}} \frac{1}{|\mathcal{Y}| \, |\Theta|} \sum_{\theta \in \Theta} \sum_{y \in \mathcal{Y}} f_{\theta}(h, y).
   \label{eq:mambr}
\end{equation}
We train evaluation metrics with different initial random seeds to generate $\Theta$, a set of diverse model parameters. Note that its concept of diversifying the internal decision of MBR decoding is similar to \citet{daheim2025uncertaintyaware}. The main difference is that we target $\theta$, but they target $\pi$ in Eq.~(\ref{eq:mbr:exp}).

\section{Empirical Analysis}
\label{sec:empirical_analysis}

We conduct empirical analysis corresponding to our theoretical analysis through experiments to comprehensively understand MBR decoding.

\subsection{Overall Settings}
\label{subsec:overall:settings}

We target three different text generation tasks, machine translation, text summarization, and image captioning to investigate the general performance of MBR decoding. In all tasks, we followed the settings of \citet{pmlr-v235-jinnai24a} for generating samples. We used epsilon sampling \citep{hewitt-etal-2022-truncation} to generate hypotheses.\footnote{Appendix \ref{app:results:hyp:sampling} includes the results with hypotheses generated by different sampling methods.} For the generation of pseudo-references, we used various sampling approaches: beam decoding, nucleus sampling \citep{Holtzman2020The} with $p=0.9$, ancestral sampling, top-$k$ sampling \citep{fan-etal-2018-hierarchical} with $k=10$, and epsilon sampling with $\epsilon=0.02$. We set the sampling size for the hypotheses to 64. We chose the sampling size for pseudo-references from \{4, 8, 16, 32, 64\}. We used the following datasets, models\footnote{We used all models from \url{https://huggingface.co/models} \citep{wolf-etal-2020-transformers}.}, and evaluation metrics for each task:

\noindent\textbf{Machine Translation} We used the WMT19 English to German (En-De) and WMT19 English to Russian (En-Ru) datasets \citep{barrault-etal-2019-findings}. We used \texttt{facebook/wmt19-en-de} for En-De and \texttt{facebook/wmt19-en-ru} for En-Ru, respectively. As the utility function and evaluation metric, we used COMET with the model \texttt{Unbabel/wmt22-comet-da}.

\noindent\textbf{Text Summarization} We used the SAMSum \citep{gliwa-etal-2019-samsum} and XSum \citep{narayan-etal-2018-dont} datasets, and used \texttt{philschmid/bart-large-cnn-samsum} and \texttt{facebook/bart-large-xsum} for generation in SAMSum and XSum, respectively. As the utility function and evaluation metric, we used BERTScore \citep{Zhang2020BERTScore} with the model \texttt{microsoft/deberta-xlarge-mnli}.

\noindent\textbf{Image Captioning} We used the MSCOCO dataset \citep{Lin2014MicrosoftCC} with the split of \citet{karpathy2015deep} and the NoCaps dataset \citep{agrawal2019nocaps}. We used \texttt{Salesforce/blip2-flan-t5-xl-coco} and \texttt{Salesforce/blip2-flan-t5-xl} for generation in MSCOCO and NoCaps, respectively. As the utility function and evaluation metric, we used BERTScore with the model \texttt{microsoft/} \texttt{deberta-xlarge-mnli}. We report the average scores on multiple references in both datasets.

Our implementation of the generation part is based on the released code of \citet{pmlr-v235-jinnai24a}\footnote{\url{https://github.com/CyberAgentAILab/model-based-mbr}}, and the MBR decoding part is based on the toolkit, \texttt{mbrs} by \citet{deguchi-2024-mbrs}\footnote{\url{https://github.com/naist-nlp/mbrs}}. We generate samples on NVIDIA GeForce RTX 3090 and perform MBR decoding on an NVIDIA RTX A6000.

\begin{figure*}[t]
    \centering
    \includegraphics[width=0.85\textwidth]{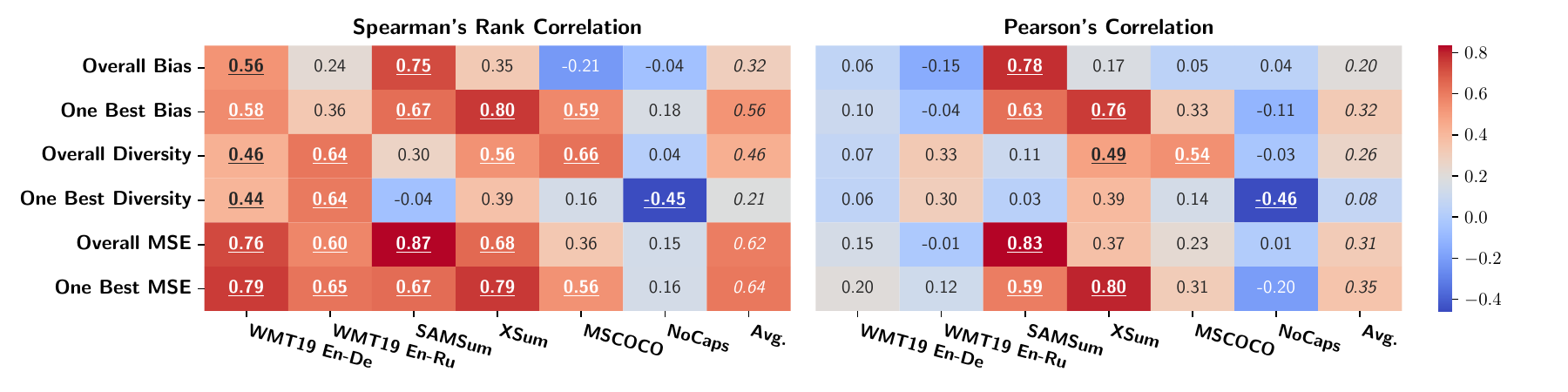}
    \caption[Correlations of each measure to the performance for each task. The underlined scores indicate statistically significant results ($p<0.05$). Note that the italic scores at \textit{Avg.} are not the target of the significance test.]{Correlations of each measure to the performance for each task. The underlined scores indicate statistically significant results ($p<0.05$).\footnotemark ~Note that the italic scores at \textit{Avg.} are not the target of the significance test.}
    \label{fig:correlation}
\end{figure*}

\begin{figure*}[t]
    \centering
    \includegraphics[width=\textwidth]{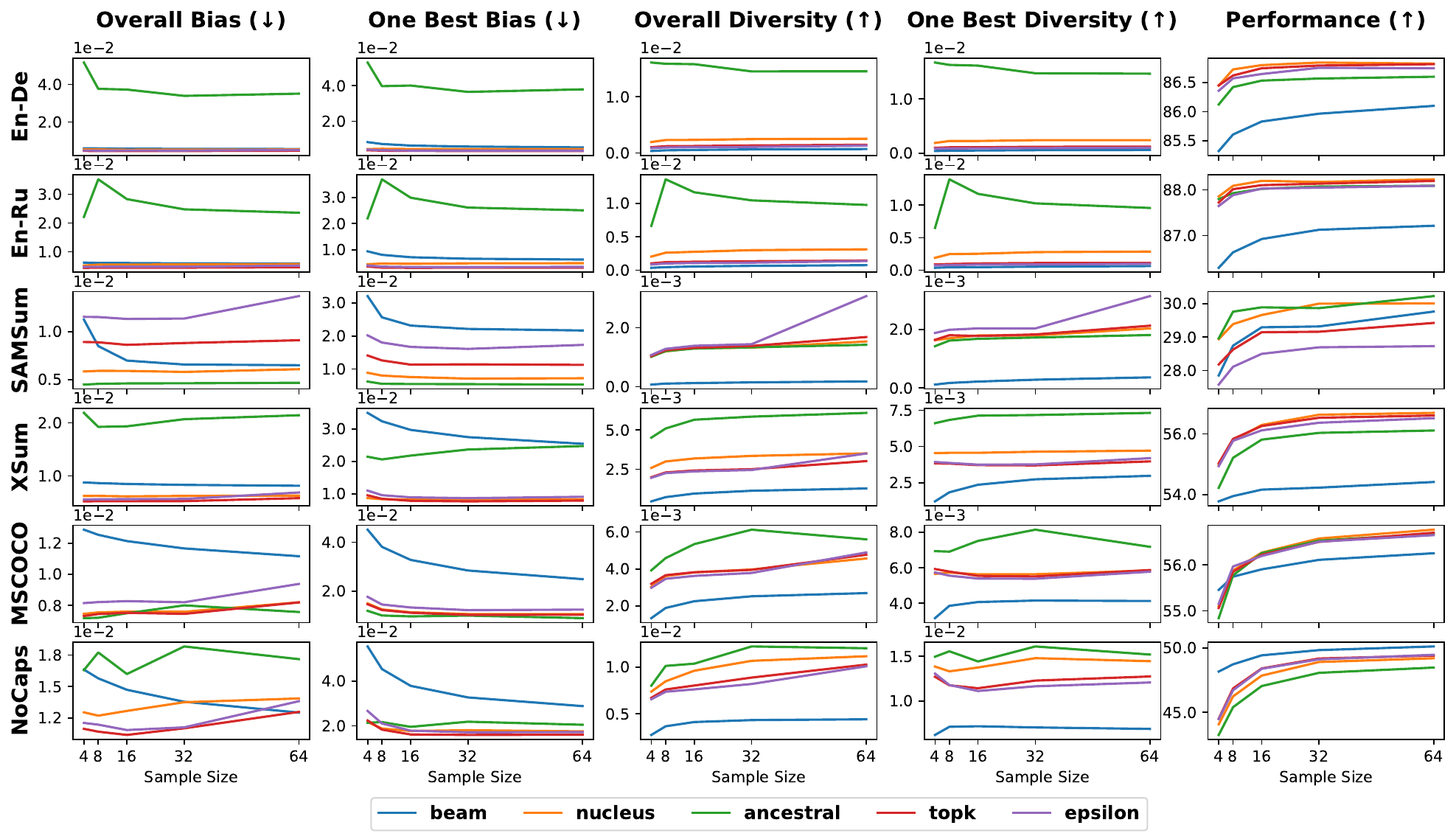}
    \caption{The relationship between bias, diversity, and performance in MBR decoding. The x-axis shows the number of used pseudo-references. ($\uparrow$) indicates higher scores are better whereas ($\downarrow$) indicates lower scores are better.}
    \label{fig:bias_diversity:full}
\end{figure*}

\subsection{Correlation of Bias and Diversity to Performance}
\label{subsec:corr:bias:diversity}

To verify our theoretical decomposition, we investigate the correlation of bias and diversity to performance on each dataset. For this purpose, we approximately compute the bias term by using our pseudo-bias in \S\ref{subsec:pseudo_bias}. Furthermore, we investigate the importance of considering the entire candidate or the best candidate.

\paragraph{Settings} We compared the following measures based on our decomposition in Eq.~(\ref{eq:bias-variance}): \textsc{Overall Bias} is \emph{Bias}; \textsc{One Best Bias} is \emph{Diveristy} for the one best result by MBR decoding; \textsc{Overall Diversity} is \emph{Diversity}; \textsc{One Best Diversity} is \emph{Diversity} for the one best result by MBR decoding; \textsc{Overall MSE} is $MSE(\hat{\textbf{u}}, \bar{\textbf{u}})$; \textsc{One Best MSE} indicates errors for the one best result by MBR decoding in $MSE(\hat{\textbf{u}}, \bar{\textbf{u}})$. For the comparison, we calculated Spearman's rank correlation and Pearson correlation between these measures and the performance based on the results of five different sampling methods with five different sampling sizes on each dataset (See \S\ref{subsec:overall:settings} for the details). Since lower bias and lower MSE are better for performance, we took their negative values in the correlation calculation. Moreover, we report averaged correlation across all datasets by Fisher z-transformation \citep{corey1998averaging}.

\paragraph{Results} Figure \ref{fig:correlation} shows the correlation between the measures and performance for each dataset. These results show that MSE for both overall and one best results correlates well with the performance for each dataset in Spearman's rank correlation, indicating the importance of considering quality estimation in MBR decoding, as in Eq.~(\ref{eq:bias-variance}). On the other hand, the decomposed bias and diversity show different tendencies. \textsc{One Best Bias}, which considers the one best result, is important for bias, whereas \textsc{Overall Diversity}, which considers overall results, is important for diversity. This result is reasonable given the assumption that MBR decoding aims to select texts that are close to human-created ones. Based on this assumption, we can say that diversity supports the selection by considering the importance of all hypotheses not covered by One Best Bias. In contrast to the results in Spearman's rank correlation, the coefficients of Pearson's correlation decrease. Based on these results, we can conclude that the measures, i.e., \textsc{One Best Bias}, \textsc{Overall Diversity}, \textsc{One Best MSE}, and \textsc{Overall MSE}, correlate well with the rank in performance, but they are challenging to capture subtle differences of values precisely. (See Appendix \ref{appendix:corr-bias-diversity} for further details.)

\begin{table*}[t]
\small
\centering
\resizebox{0.75\textwidth}{!}{
\begin{tabular}{llcccccccccc}
\toprule
                                &   & \multicolumn{5}{c}{WMT19 En-De}  & \multicolumn{5}{c}{WMT19 En-Ru}  \\
                                \cmidrule{1-2}\cmidrule(l){3-7}\cmidrule(l){8-12}
                            \multicolumn{2}{c}{Num. of Samples} & 4    & 8    & 16   & 32   & 64   & 4    & 8    & 16   & 32   & 64   \\\cmidrule{1-2}\cmidrule(l){3-7}\cmidrule(l){8-12}
\multirow{4}{*}{Num. of Models}  & 1 & 85.7          & 85.9          & 85.9          & 85.9          & 85.9 & \textbf{87.4} & 87.4          & 87.5          & 87.5          & 87.5          \\\cmidrule{2-2}\cmidrule(l){3-7}\cmidrule(l){8-12}
 & 2 & 85.7          & \textbf{86.0} & \textbf{86.0} & 85.9          & 85.9 & \textbf{87.4} & 87.4          & 87.5          & 87.5          & \textbf{87.6} \\
 & 4 & \textbf{85.8} & \textbf{86.0} & \textbf{86.0} & \textbf{86.0} & 86.0 & \textbf{87.4} & 87.4          & 87.5          & 87.5          & \textbf{87.6} \\
 & 8 & \textbf{85.8} & \textbf{86.0} & \textbf{86.0} & \textbf{86.0} & \textbf{86.1} & \textbf{87.4} & \textbf{87.5} & \textbf{87.6} & \textbf{87.6} & \textbf{87.6} \\
\midrule
                                &   & \multicolumn{5}{c}{SAMSum}  & \multicolumn{5}{c}{XSum}  \\
                                \cmidrule{1-2}\cmidrule(l){3-7}\cmidrule(l){8-12}
                            \multicolumn{2}{c}{Num. of Samples} & 4    & 8    & 16   & 32   & 64   & 4    & 8    & 16   & 32   & 64   \\\cmidrule{1-2}\cmidrule(l){3-7}\cmidrule(l){8-12}
\multirow{4}{*}{Num. of Models} & 1 & 28.6          & 29.1          & 29.5          & 29.5          & 29.7          & \textbf{54.2} & 55.2          & 55.7          & 56.0          & 56.1          \\\cmidrule{2-2}\cmidrule(l){3-7}\cmidrule(l){8-12}
 & 2 & 28.8          & \textbf{29.6} & \textbf{29.9} & \textbf{29.9} & 30.1          & \textbf{54.2} & 55.2          & 55.7          & \textbf{56.1} & \textbf{56.2} \\
 & 4 & \textbf{28.7} & 29.5          & \textbf{29.9} & 29.8          & \textbf{30.2} & \textbf{54.2} & 55.2          & \textbf{55.8} & \textbf{56.1} & \textbf{56.2} \\
 & 8 & \textbf{28.7} & 29.5          & 29.8          & \textbf{29.9} & 30.1          & \textbf{54.3} & \textbf{55.3} & \textbf{55.8} & \textbf{56.1} & \textbf{56.2} \\
\midrule
                                &   & \multicolumn{5}{c}{MSCOCO}  & \multicolumn{5}{c}{NoCaps}  \\
                                \cmidrule{1-2}\cmidrule(l){3-7}\cmidrule(l){8-12}
                            \multicolumn{2}{c}{Num. of Samples} & 4    & 8    & 16   & 32   & 64   & 4    & 8    & 16   & 32   & 64   \\\cmidrule{1-2}\cmidrule(l){3-7}\cmidrule(l){8-12}
\multirow{4}{*}{Num. of Models}  & 1 & \textbf{54.9} & 55.8          & 56.3          & 56.5          & 56.8          & 42.9          & 45.3          & 46.8          & 47.8          & 48.6          \\\cmidrule{2-2}\cmidrule(l){3-7}\cmidrule(l){8-12}
 & 2 & \textbf{54.9} & 55.8          & 56.4          & 56.6          & 56.8          & 43.2          & 45.6          & 47.2          & 48.3          & 48.9          \\
 & 4 & \textbf{54.9} & \textbf{56.0} & 56.4          & 56.7          & \textbf{56.9} & 43.3          & 45.6          & 47.3          & 48.4          & \textbf{49.0} \\
 & 8 & \textbf{54.9} & \textbf{56.0} & \textbf{56.5} & \textbf{56.8} & \textbf{56.9} & \textbf{43.5} & \textbf{45.7} & \textbf{47.4} & \textbf{48.5} & \textbf{49.0}\\
\bottomrule
\end{tabular}
}
\caption{Results of MAMBR with ancestral sampling. Bold font indicates the best result.}
\label{tab:mambr_ancestral_full}
\end{table*}

\footnotetext{We used Student's t-test \citep{student1908probable}.} 

\subsection{Bias and Diversity Trade-off}
\label{subsec:experiments:bias-diversity}

To investigate the bias-diversity trade-off in more detail, we followed the setup described in \S\ref{subsec:corr:bias:diversity}. We plotted the results for each dataset using different sampling methods in Figure \ref{fig:bias_diversity:full}. The results support the bias-diversity trade-off shown in Theorems \ref{theorem:bias-variance-tradeoff-1} and \ref{theorem:bias-variance-tradeoff-2}. As a case study, while ancestral sampling exhibits the highest bias, except in the case of the SAMSum dataset, it sometimes outperforms other sampling methods owing to its greater diversity. Focusing on top-k sampling, which has the lowest bias, again excluding the SAMSum dataset, we can observe that the reduction in bias tends to limit the increase in diversity. This finding supports our previously noted bias-diversity trade-off in MBR decoding. However, as evidenced by the performance of beam decoding, which has the lowest diversity, the importance of bias and diversity varies depending on the target dataset. Therefore, while our theoretical analysis effectively explains the performance tendencies in MBR decoding, it remains essential to consider task-specific features carefully to achieve further performance improvements. (See Appendix \ref{appendix:bias-diversity-tradeoff} for further details.)

\subsection{Inference Scaling Laws}
\label{subsec:inference_scaling_laws}

Figure \ref{fig:bias_diversity:full} indicates improved performance by increasing the sample size, in line with Theorems \ref{theorem:submodularity:terms} and \ref{theorem:submodularity:bounds}. Since MBR decoding only considers the pairs of a pseudo-reference and hypothesis for scoring, the interaction between pseudo-references is not directly considered. This goes along with the assumption of these theorems, the conditional independence of the scores of utility functions. This is because Figure \ref{fig:bias_diversity:full} shows the performance improvement that our theoretical analysis can explain. Note that there is a possibility that inference methods with interaction between samples, such as universal self-consistency, do not show performance improvement like this due to the same reason.

\subsection{Effectiveness of Metric-augmented MBR}
\label{subsec:experiments:mambr}

We investigate the possibility of improving the performance of MAMBR in Eq.~(\ref{eq:mambr}) by changing the automatic evaluation metric's model parameters.

\paragraph{Settings} To prepare the set of model parameters, we trained eight models by varying their initial seeds. We trained \texttt{Unbabel/wmt22-comet-da} on the Direct Assessments (DA) task \citep{graham-etal-2013-continuous}, using the WMT 2017 to 2020 datasets \citep{bojar-etal-2017-findings,bojar-etal-2018-findings,barrault-etal-2019-findings,barrault-etal-2020-findings} for training and the WMT 2021 dataset \citep{akhbardeh-etal-2021-findings} for validation in COMET. Additionally, we trained \texttt{microsoft/deberta-large} on the MNLI dataset from GLUE \citep{wang-etal-2018-glue} for BERTScore. During inference, to control model diversity, we selected the top-$n$ models based on their proximity to the median validation scores, with $n$ chosen from {1, 2, 4, 8}. For the generation, we used ancestral sampling.

\paragraph{Results} Tables~\ref{tab:mambr_ancestral_full} show the MAMBR results. We observe performance improvement as the number of models increases. This suggests that MAMBR can improve performance by enhancing the diversity of evaluation metrics along with our theoretical insights. (Appendix \ref{appendix:bias-diversity-mambr} includes further details.)

\section{Conclusion}

This work provides a unified theoretical interpretation of Minimum Bayes Risk (MBR) decoding through the lens of bias-diversity decomposition. By decomposing the errors in quality estimation in MBR decoding into bias and diversity, we highlight the trade-off between improving these two factors, with an emphasis on the benefits of increasing diversity, which is behind the inference scaling laws. Our theoretical insights align with previous empirical results, and we further investigate aspects not covered by these empirical findings through the introduction of the pseudo-bias metric and MAMBR decoding. Experimental results across multiple tasks demonstrate the validity of our theoretical findings and the effectiveness of our approach in improving text generation quality. These findings bridge the gap between empirical observations and theoretical understanding of MBR decoding, offering new insights for optimizing text generation.

\section{Limitation}
\label{app:limitation}

Unlike the decomposition based on information-theoretic diversity, the bias-diversity decomposition for MBR decoding does not theoretically explain the diversity of the hypotheses and their aligned model-side behaviors. Corresponding to this limitation, we conduct a limited empirical analysis presented in Appendix \ref{app:results:hyp:sampling}, similar to previous works \citep{eikema-aziz-2020-map, fernandes-etal-2022-quality, freitag-etal-2023-epsilon}.

\section{Ethical Consideration}

We used GPT-4o from OpenAI in writing to check grammatical errors.

\section*{Acknowledgments}
We are grateful to Graham Neubig for introducing his team's work \citep{bertsch-etal-2023-mbr}, which reveals MBR decoding as a fundamental approach covering various decoding methods. The existence of this paper motivated us to explore MBR decoding, as we have done in this work. This work was supported by JSPS KAKENHI Grant Number JP23H03458.

\bibliography{custom}

\newpage
\onecolumn
\appendix

\section{Proofs}
\subsection{Proof for Theorem \ref{theorem:bias-variance}}
\label{appendix:proof-bias-variance}
First, we can decompose $(\hat{u}_{i} - \bar{u}_{i})^2$ as follows:
\begin{align}
&(\hat{u}_{i} - \bar{u}_{i})^2\\
=&(\hat{u}_{i})^2 - 2\hat{u}_{i}\bar{u}_{i} + (\bar{u}_{i})^2\\
=&(\hat{u}_{i})^2 - 2\hat{u}_{i}\bar{u}_{i} + 2(\bar{u}_{i})^2 - (\bar{u}_{i})^2\\
=&(\hat{u}_{i})^2 - 2\hat{u}_{i}\bar{u}_{i} + 2\bar{u}_{i}\bar{u}_{i} - (\bar{u}_{i})^2\\
=&(\hat{u}_{i})^2 - \frac{1}{|\mathcal{Y}|}\sum_{j=1}^{|\mathcal{Y}|}2\hat{u}_{i}u_{i}^{j} + \frac{1}{|\mathcal{Y}|}\sum_{j=1}^{|\mathcal{Y}|}2\bar{u}_{i}u_{i}^{j} - (\bar{u}_{i})^2\\
=&\frac{1}{|\mathcal{Y}|}\sum_{j=1}^{|\mathcal{Y}|}(\hat{u}_{i})^2 - \frac{1}{|\mathcal{Y}|}\sum_{j=1}^{|\mathcal{Y}|}2\hat{u}_{i}u_{i}^{j} + \frac{1}{|\mathcal{Y}|}\sum_{j=1}^{|\mathcal{Y}|}2\bar{u}_{i}u_{i}^{j} - \frac{1}{|\mathcal{Y}|}\sum_{j=1}^{|\mathcal{Y}|}(\bar{u}_{i})^2\\
=&\frac{1}{|\mathcal{Y}|}\sum_{j=1}^{|\mathcal{Y}|}((\hat{u}_{i})^2 - 2\hat{u}_{i}u_{i}^{j} + 2\bar{u}_{i}u_{i}^{j} - (\bar{u}_{i})^2)\\
=&\frac{1}{|\mathcal{Y}|}\sum_{j=1}^{|\mathcal{Y}|}((\hat{u}_{i})^2 - 2\hat{u}_{i}u_{i}^{j} + (u_{i}^{j})^2 - (u_{i}^{j})^2 + 2\bar{u}_{i}u_{i}^{j} - (\bar{u}_{i})^2)\\
=&\frac{1}{|\mathcal{Y}|}\sum_{j=1}^{|\mathcal{Y}|}((\hat{u}_{i})^2 - 2\hat{u}_{i}u_{i}^{j} + (u_{i}^{j})^2 - ((u_{i}^{j})^2 - 2\bar{u}_{i}u_{i}^{j} + (\bar{u}_{i})^2))\\
=&\frac{1}{|\mathcal{Y}|}\sum_{j=1}^{|\mathcal{Y}|}((\hat{u}_{i} - u_{i}^{j})^2 - (\bar{u}_{i} - u_{i}^{j})^2)\\
=& \frac{1}{|\mathcal{Y}|}\sum_{j=1}^{|\mathcal{Y}|}(\hat{u}_{i} - f_{\mathbf{\theta}}(h_i,y_j))^2 - \frac{1}{|\mathcal{Y}|}\sum_{j=1}^{|\mathcal{Y}|}(\bar{u}_{i} - f_{\mathbf{\theta}}(h_i,y_j))^2
\label{eq:app:decompose:each}
\end{align}
By utilizing Eq.~(\ref{eq:app:decompose:each}), we can further decompose $MSE(\hat{\textbf{u}}, \bar{\textbf{u}})$ as follows:
\begin{align}
&MSE(\hat{\textbf{u}}, \bar{\textbf{u}})\\
=&\frac{1}{|\mathcal{H}|}\sum_{i=1}^{|\mathcal{H}|}(\hat{u}_{i} - \bar{u}_{i})^2\\
=&\frac{1}{|\mathcal{H}|}\sum_{i=1}^{|\mathcal{H}|}\Bigl(\frac{1}{|\mathcal{Y}|}\sum_{j=1}^{|\mathcal{Y}|}(\hat{u}_{i} - f_{\mathbf{\theta}}(h_i,y_j))^2 - \frac{1}{|\mathcal{Y}|}\sum_{j=1}^{|\mathcal{Y}|}(\bar{u}_{i} - f_{\mathbf{\theta}}(h_i,y_j))^2\Bigr)\\
=&	\underbrace{\frac{1}{|\mathcal{H}|}\sum_{i=1}^{|\mathcal{H}|}\frac{1}{|\mathcal{Y}|}\sum_{j=1}^{|\mathcal{Y}|}(\hat{u}_{i} - f_{\mathbf{\theta}}(h_i,y_j))^2}_{\text{Bias}} - \underbrace{\frac{1}{|\mathcal{H}|}\sum_{i=1}^{|\mathcal{H}|}\frac{1}{|\mathcal{Y}|}\sum_{j=1}^{|\mathcal{Y}|}(\bar{u}_{i} - f_{\mathbf{\theta}}(h_i,y_j))^2}_{\text{Diversity}}
\label{eq:app:decompose:all}
\end{align}
Finally, Theorem~\ref{theorem:bias-variance} is proved.

\subsection{Proof for Theorem \ref{theorem:bias-variance-tradeoff-1}}
\label{appendix:proof-bias-variance-tradeoff-1}

In Eq.~(\ref{eq:bias-variance}), when \emph{Bias} becomes zero, $f_{\mathbf{\theta}}(h_i,y_j)$ becomes $\hat{u}_{i}$ in any $j$. Because $\bar{u}_{i}$ is an average of $f_{\mathbf{\theta}}(h_i,y_j)$ for all $j$, $\bar{u}_{i}$ becomes $\hat{u}_{i}$ in this condition. Since that means $(\bar{u}_{i} - f_{\mathbf{\theta}}(h_i,y_j))^2$ becomes zero and then \emph{Diversity} becomes zero. Therefore, Theorem~\ref{theorem:bias-variance-tradeoff-1} is proved.

\subsection{Proof for Theorem \ref{theorem:bias-variance-tradeoff-2}}
\label{appendix:proof-bias-diversity-dependence}
See \citet{JMLR:v6:brown05a} for the proof of the decomposition. Since this decomposition is applicable when Eq.~(\ref{eq:bias-variance}) holds, we can apply this decomposition to our analysis.

\subsection{Proof for Theorem \ref{theorem:bias-variance-general}}
\label{appendix:proof-it-diversity}
See \citet{zhou2010} for the decomposition based on the information-theoretic diversity. Since there is no restriction on the variables used, we can apply their decomposition to our analysis.

\subsection{Proof for Theorem \ref{theorem:monotonicity-terms}} 
\label{appendix:proof-monotonicity-terms}

See \citet{kamigaito2025arxiv} for the proof. Since the monotonicity is generally applicable without limitation, we can apply their proof to our analysis.

\subsection{Proof for Theorem \ref{theorem:monotonicity-bounds}}
\label{appendix:monotonicity-bounds}

See \citet{kamigaito2025arxiv} for the proof. Similar to Appendix \ref{appendix:proof-monotonicity-terms}, we can apply their proof to our analysis.

\subsection{Proof for Theorem \ref{theorem:submodularity:terms}}
\label{appendix:submodularity:terms}

See \citet{kamigaito2025arxiv} for the proof. Since conditional independence is an assumption, we can apply their proof to our analysis.

\subsection{Proof for Theorem \ref{theorem:submodularity:bounds}}
\label{appendix:submodularity:bounds}

See \citet{kamigaito2025arxiv} for the proof. Similar to Appendix \ref{appendix:submodularity:terms}, we can apply their proof to our analysis.

\section{Interpretation as Ensemble Learning}
\label{appendix:interpretation:ensemble-learning}

When $|\mathcal{Y}|$ is large enough to satisfy the law of large numbers, we can induce the following expectation in MBR decoding by using a model's prediction, $P(y|x)$:
\begin{align}
   & \argmax_{h \in \mathcal{H}}\sum_{y \in \Omega}f_{\mathbf{\theta}}(h,y)P(y|x) \label{eq:mbr:ensemble-learning:ideal}
   \\
   = & \argmax_{h \in \mathcal{H}}\mathbb{E}_{P(y|x)}[f_{\mathbf{\theta}}(h,y)] \label{eq:mbr:ensemble-learning:expectation}\\
   \approx &\argmax_{h \in \mathcal{H}}\frac{1}{|\mathcal{Y}|}\sum_{y \in \mathcal{Y}}f_{\mathbf{\theta}}(h,y),\:\:\:y_1, \cdots, y_{|\mathcal{Y}|} \sim P(y|x) \label{eq:mbr:ensemble-learning:sampling}
\end{align}
Since this expectation is based on $P(y|x)$, we can understand the importance of increasing the number of pseudo-references to induce a reliable $P(y|x)$.

\begin{theorem}
When $f_{\mathbf{\theta}}(h,y)$ is normalized as a probability $P_{\mathbf{\theta}}(h|y)$, Eq.~(\ref{eq:mbr:ensemble-learning:ideal}) is equivalent to the Bayes Optimal Classifier (BOC) in \citet{mitchell1997introduction}.
\label{theorem:is:boc}
\end{theorem}
\begin{proof}
Self-evident by the following reformulation:
    \begin{equation}
        \argmax_{h \in \mathcal{H}}\sum_{y \in \Omega}f_{\mathbf{\theta}}(h,y)P(y|x) \\
        = \argmax_{h \in \mathcal{H}}\sum_{y \in \Omega}P_{\mathbf{\theta}}(h|y)P(y|x)
        \label{eq:boc}
    \end{equation}
\end{proof}
\begin{theorem}
When $f_{\mathbf{\theta}}(h,y)$ is normalized as a probability $P_{\mathbf{\theta}}(h|y)$, Eq.~(\ref{eq:mbr:ensemble-learning:sampling}) is equivalent to the Gibbs algorithm in \citet{mitchell1997introduction} that approximates BOC by sampling.
\label{theorem:is:gibbs}
\end{theorem}
\begin{proof} Self-evident by the following reformulation:
    \begin{align}
        & \argmax_{h \in \mathcal{H}}\frac{1}{|\mathcal{Y}|}\sum_{y \in \mathcal{Y}}f_{\mathbf{\theta}}(h,y),\:\:\:y_1, \cdots, y_{|\mathcal{Y}|} \sim P(y|x) \\
        = & \argmax_{h \in \mathcal{H}}\sum_{y \in \mathcal{Y}}P_{\mathbf{\theta}}(h|y),\:\:\:y_1, \cdots, y_{|\mathcal{Y}|} \sim P(y|x)
        \label{eq:gibbs}
    \end{align}
\end{proof}
Hence, we can understand that MBR decoding represented as Eqs.~(\ref{eq:mbr:ensemble-learning:ideal}) and (\ref{eq:mbr:ensemble-learning:sampling}) approximates the ensemble learning method, BOC. In this interpretation, since $P(y|x)$ is a prior of BOC, we can also understand that MBR approximately uses the model-predicted probability as its prior.

When pruning unnecessary $y$ in the BOC formulation of Eq.~(\ref{eq:boc}), because the sum of $P_{\mathbf{\theta}}(h|y)$ for all $h$ is always 1, we can determine the importance of $y$ based solely on $P(y|x)$. Since we can arbitrarily choose $P(y|x)$ during sampling, we understand that pruning methods select the importance of each $y$ as a prior in BOC. Note that utility functions are not always normalized; therefore, there is a gap between this interpretation and the actual MBR decoding. Addressing this gap remains an open problem.

In practice, directly drawing samples from $P(y|x)$ is intractable. Therefore, we must use approximate search methods, which are commonly influenced by left-to-right decoding and threshold values. These factors can lead to unreachable states and biases, as seen in greedy or beam decoding and other sampling approaches. Letting $P'(y|x)$ denote the model's prediction with the approximate search, we can similarly induce the following expectation:
\begin{align}
&\argmax_{h \in \mathcal{H}}\frac{1}{|\mathcal{Y}|}\sum_{y \in \mathcal{Y}}f_{\mathbf{\theta}}(h,y),\:\:\:y_1, \cdots, y_{|\mathcal{Y}|} \sim P'(y|x) \\
    \approx &\argmax_{h \in \mathcal{H}}\mathbb{E}_{P'(y|x)}[f_{\mathbf{\theta}}(h,y)]
    \label{eq:mbr:ensemble-learning:skewed}
\end{align}
Unfortunately, due to $P'(y|x)$, Eq.~(\ref{eq:mbr:ensemble-learning:skewed}) deviates from Eq.~(\ref{eq:mbr:ensemble-learning:expectation}). To precisely predict Eq.~(\ref{eq:mbr:ensemble-learning:ideal}) using samples from $P'(y|x)$, we can consider the following theorem:
\begin{theorem}
When $|\mathcal{Y}|$ is large enough to satisfy the law of large numbers, by using importance sampling, we can induce Eq.~(\ref{eq:mbr:ensemble-learning:expectation}) from $P'(y|x)$.
\label{theorem:is:mbr}
\end{theorem}
\begin{proof}
\begin{align}
   & \argmax_{h \in \mathcal{H}}\mathbb{E}_{P(y|x)}[f_{\mathbf{\theta}}(h,y)]\\
   = & \argmax_{h \in \mathcal{H}}\sum_{y}P(y|x)f_{\mathbf{\theta}}(h,y)\\
   = & \argmax_{h \in \mathcal{H}}\sum_{y}P(y|x)f_{\mathbf{\theta}}(h,y)\frac{P'(y|x)}{P'(y|x)}\\
   = & \argmax_{h \in \mathcal{H}}\sum_{y}P'(y|x)f_{\mathbf{\theta}}(h,y)\frac{P(y|x)}{P'(y|x)}\\
   \approx & \argmax_{h \in \mathcal{H}}\sum_{y \in \mathcal{Y}}f_{\mathbf{\theta}}(h,y)\frac{P(y|x)}{P'(y|x)},\:\:\:y_1, \cdots, y_{|\mathcal{Y}|} \sim P'(y|x) 
   \label{eq:mbr:ensemble-learning:importance-sampling}
\end{align}
\end{proof}
Apart from the fact that even precisely calculating $P'(y|x)$ is also difficult, we can induce the following theorem:
\begin{theorem}
When $|\mathcal{Y}|$ is large enough to satisfy the law of large numbers and $P'(y|x)$ equals a discrete uniform distribution $\mathcal{U}(0,|\mathcal{Y}|)$, Eq.~(\ref{eq:mbr:ensemble-learning:importance-sampling}) is equivalent to Model-based MBR (MBMBR) of \citet{pmlr-v235-jinnai24a}.
\label{theorem:is:mbmbr}
\end{theorem}
\begin{proof}
    \begin{align}
        & \argmax_{h \in \mathcal{H}}\sum_{y \in \mathcal{Y}}f_{\mathbf{\theta}}(h,y)\frac{P(y|x)}{P'(y|x)},\:\:\:y_1, \cdots, y_{|\mathcal{Y}|} \sim P'(y|x) \\
        = & \argmax_{h \in \mathcal{H}}\frac{1}{|\mathcal{Y}|}\sum_{y \in \mathcal{Y}}f_{\mathbf{\theta}}(h,y)P(y|x),\:\:\:y_1, \cdots, y_{|\mathcal{Y}|} \sim \mathcal{U}(0,|\mathcal{Y}|) \\
        = & \argmax_{h \in \mathcal{H}}\sum_{y \in \mathcal{Y}}f_{\mathbf{\theta}}(h,y)P(y|x),\:\:\:y_1, \cdots, y_{|\mathcal{Y}|} \sim \mathcal{U}(0,|\mathcal{Y}|) \label{eq:mbmbr}
    \end{align}
\end{proof}

From Theorem \ref{theorem:is:mbmbr}, we can understand that MBMBR is an effective approach when sampling methods are unreliable. Based on the interpretation from the viewpoint of BOC, Eq.~(\ref{eq:mbmbr}) estimates the importance for each $y$ through prior $P(y|x)$, which can be used for pruning $y$.

Even though our interpretation can explain the pruning of pseudo-references based on priors in BOC, pruning hypotheses is out of scope for this interpretation.

\section{Correlation of Bias and Diversity to Performance}
\label{appendix:corr-bias-diversity}

We further investigate whether our analysis in \S\ref{subsec:corr:bias:diversity} is consistent when metrics used in MBR decoding and performance evaluation are different.

\paragraph{Settings} Based on the inherited settings from \S\ref{subsec:corr:bias:diversity}, we changed the performance evaluation metrics, COMET and BERTScore to BLEURT \citep{sellam-etal-2020-bleurt} and chrF++ \citep{popovic-2015-chrf,popovic-2017-chrf}. We used BLEURT on single-sentence generation tasks, WMT19 En-De and En-Ru, XSum, MSCOCO, and NoCaps. Since SAMSum is a multiple-sentence generation task and BLEURT cannot handle it, we used chrF++ instead. 
\paragraph{Results}
Figure \ref{fig:correlation:others} shows the correlation. Similar to the results in \S\ref{subsec:corr:bias:diversity}, the measures, i.e., \textsc{One Best Bias}, \textsc{Overall Diversity}, and \textsc{One Best MSE} in Spearman's rank correlation correlate well with the rank in performance, even though these correlation values are degraded by different evaluation metrics from decoding time. The lower correlation values in Pearson's correlation than Spearman's rank correlation also show similar tendencies in \S\ref{subsec:corr:bias:diversity} and indicate the difficulty of precisely estimating the performance values from these measures. From these results, we can confirm that correlation tendencies are consistent when changing the performance evaluation metrics.

\begin{figure*}[t]
    \centering
    \includegraphics[width=\textwidth]{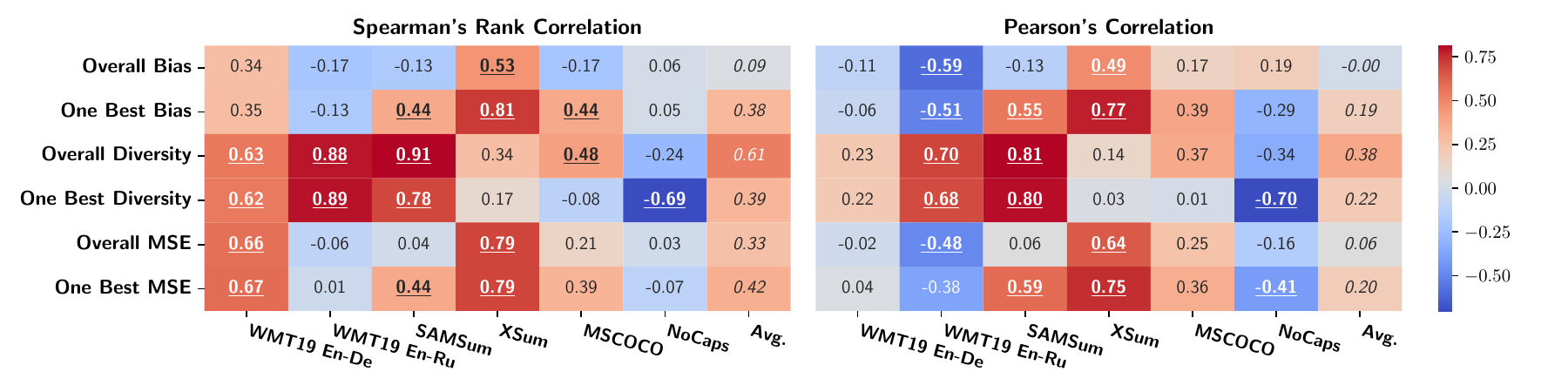}
    \caption{Correlation between measures in our decomposition and performance for each dataset when using different metrics in decoding and performance evaluation. The notations are the same as Figure \ref{fig:correlation}.}
    \label{fig:correlation:others}
\end{figure*}

\begin{figure*}[t]
    \centering
    \includegraphics[width=0.975\textwidth]{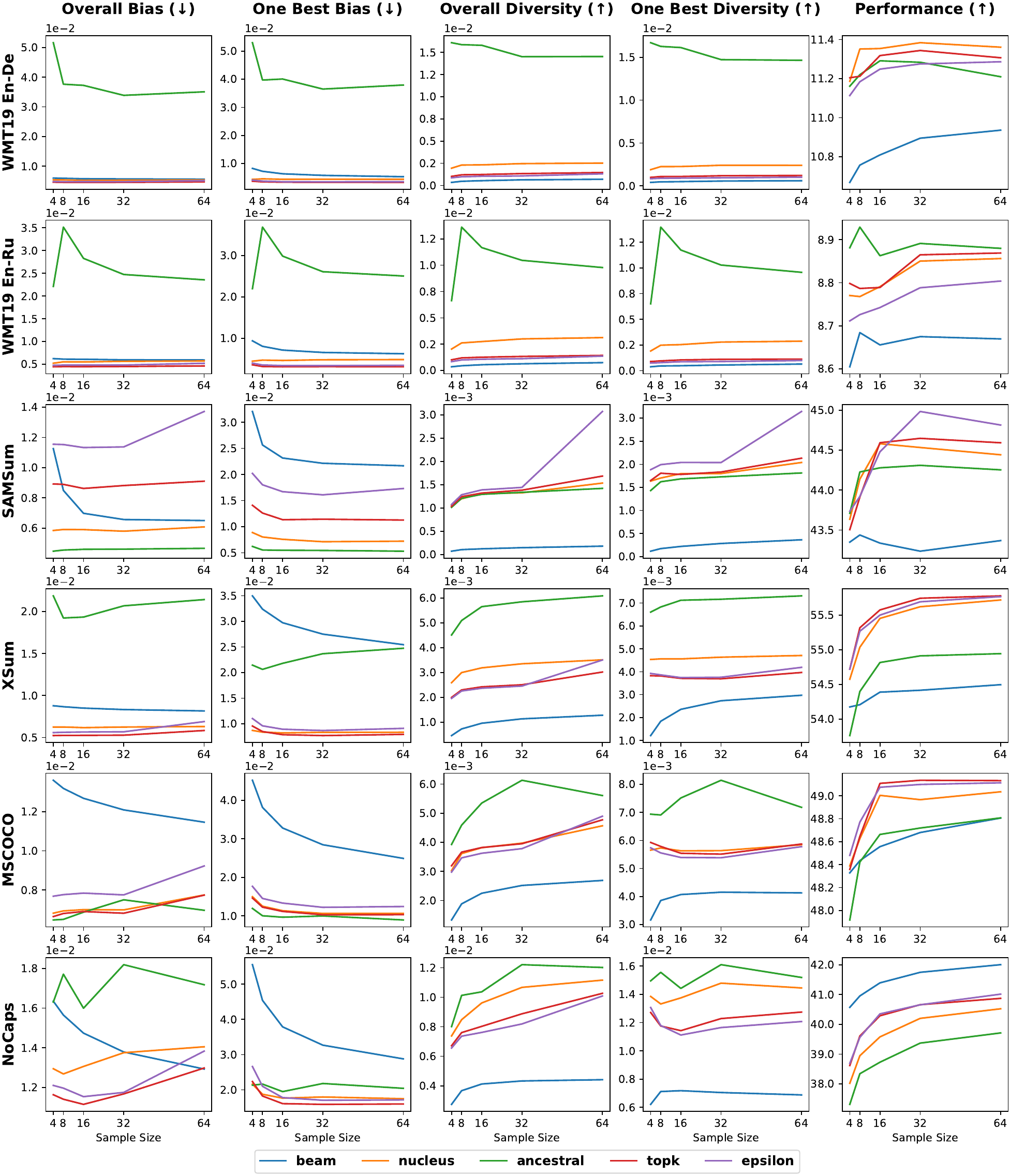}
    \caption{The relationship between bias, diversity, and performance in MBR decoding when using different metrics in decoding and performance evaluation. The notations are the same as Figure \ref{fig:bias_diversity:full}.}
    \label{fig:bias_diversity:full:others}
\end{figure*}

\section{Bias and Diversity Trade-off}
\label{appendix:bias-diversity-tradeoff}

Similar to Appendix \ref{appendix:corr-bias-diversity}, we further investigate whether our analysis in \S\ref{subsec:experiments:bias-diversity} is consistent when metrics used in MBR decoding and performance evaluation are different.

\paragraph{Settings} We inherited the setting of Appendix \ref{appendix:corr-bias-diversity}. Thus, COMET and BERTScore used in MBR decoding are replaced with BLEURT and chrF++ in performance evaluation.

\paragraph{Results} Figure \ref{fig:bias_diversity:full:others} shows the results. We can see the changed performances in the subfigures of the rightmost column. The entire tendencies of beam decoding are almost the same as Figure \ref{fig:bias_diversity:full}, excluding the case of the performance drop in SAMSum, whose evaluation metric is changed from BERTScore to chrF++. However, this behavior is reasonable considering the highest \texttt{One Best Bias} and lowest \texttt{Overall Diversity} of beam decoding in SAMSum. This result shows the possibility of adopting bias and diversity in a metric to estimate performance in other evaluation metrics. On the other hand, these relationships are not always consistent, as represented by the uncorrelated values on NoCaps that permit diversified generation, as shown by its 10 gold references.

\begin{table*}[t]
\small
\centering
\resizebox{0.88\textwidth}{!}{
\begin{tabular}{llrrrrrrrrrr}
\toprule
&   & \multicolumn{5}{c}{WMT19 En-De}  & \multicolumn{5}{c}{WMT19 En-Ru}  \\
\cmidrule{1-2}\cmidrule(l){3-7}\cmidrule(l){8-12}
\multicolumn{2}{c}{Num. of Samples} & \multicolumn{1}{c}{4}    & \multicolumn{1}{c}{8}    & \multicolumn{1}{c}{16}   & \multicolumn{1}{c}{32}   & \multicolumn{1}{c}{64}   & \multicolumn{1}{c}{4}    & \multicolumn{1}{c}{8}    & \multicolumn{1}{c}{16}   & \multicolumn{1}{c}{32}   & \multicolumn{1}{c}{64}   \\\cmidrule{1-2}\cmidrule(l){3-7}\cmidrule(l){8-12}
\multirow{4}{*}{Num. of Models}  & 1 & 238 & 211 & 209 & 221 & 274 & 214 & 212 & 216 & 236 & 282 \\\cmidrule{2-2}\cmidrule(l){3-7}\cmidrule(l){8-12}
 & 2 & 341 & 466 & 330 & 359 & 472 & 324 & 331 & 341 & 367 & 457 \\
 & 4 & 592 & 554 & 682 & 715 & 754 & 562 & 608 & 578 & 621 & 866 \\
 & 8 & 1,129 & 1,067 & 1,059 & 1,119 & 1,549 & 1,205 & 1,018 & 1,014 & 1,158 & 1,531 \\
\midrule
&   & \multicolumn{5}{c}{SAMSum}  & \multicolumn{5}{c}{XSum}  \\
\cmidrule{1-2}\cmidrule(l){3-7}\cmidrule(l){8-12}
\multicolumn{2}{c}{Num. of Samples} & \multicolumn{1}{c}{4}    & \multicolumn{1}{c}{8}    & \multicolumn{1}{c}{16}   & \multicolumn{1}{c}{32}   & \multicolumn{1}{c}{64}   & \multicolumn{1}{c}{4}    & \multicolumn{1}{c}{8}    & \multicolumn{1}{c}{16}   & \multicolumn{1}{c}{32}   & \multicolumn{1}{c}{64}   \\\cmidrule{1-2}\cmidrule(l){3-7}\cmidrule(l){8-12}
\multirow{4}{*}{Num. of Models} & 1 & 390 & 424 & 544 & 748 & 1,093  & 3,612 & 4,249 & 4,451 & 5,947 & 10,936 \\\cmidrule{2-2}\cmidrule(l){3-7}\cmidrule(l){8-12}
 & 2 & 592 & 717 & 902 & 1,300 & 1,945 & 5,267 & 5,978 & 7,201 & 10,478 & 20,404 \\
 & 4 & 1,149 & 1,122 & 1,388 & 1,994 & 4,005 & 8,805 & 11,584 & 15,265 & 19,438 & 38,137 \\
 & 8 & 1,838 & 2,100 & 2,626 & 3,764 & 7,821 & 16,397 & 19,392 & 28,844 & 37,442 & 60,800 \\
\midrule
&   & \multicolumn{5}{c}{MSCOCO}  & \multicolumn{5}{c}{NoCaps}  \\
\cmidrule{1-2}\cmidrule(l){3-7}\cmidrule(l){8-12}
\multicolumn{2}{c}{Num. of Samples} & \multicolumn{1}{c}{4}    & \multicolumn{1}{c}{8}    & \multicolumn{1}{c}{16}   & \multicolumn{1}{c}{32}   & \multicolumn{1}{c}{64}   & \multicolumn{1}{c}{4}    & \multicolumn{1}{c}{8}    & \multicolumn{1}{c}{16}   & \multicolumn{1}{c}{32}   & \multicolumn{1}{c}{64}   \\\cmidrule{1-2}\cmidrule(l){3-7}\cmidrule(l){8-12}
\multirow{4}{*}{Num. of Models}  & 1 & 1,209 & 1,143 & 1,582 & 1,906 & 2,849  & 1,235 & 1,179 & 1,406 & 1,831 & 3,453 \\\cmidrule{2-2}\cmidrule(l){3-7}\cmidrule(l){8-12}
 & 2 & 1,739 & 1,772 & 2,580 & 3,261 & 5,144 & 1,574 & 1,647 & 2,172 & 2,935 & 5,636 \\
 & 4 & 2,439 & 3,035 & 4,541 & 5,986 & 9,758 & 2,644 & 3,330 & 4,170 & 5,201 & 8,916 \\
 & 8  & 4,316 & 6,099 & 7,557 & 11,397 & 19,108 & 3,784 & 5,432 & 6,666 & 10,210 & 17,052\\
\bottomrule
\end{tabular}
}
\caption{Time usages (seconds) by MAMBR in each setting with ancestral sampling.
}
\label{tab:mambr_time}
\end{table*}
\begin{table*}[t]
\small
\centering
\resizebox{0.88\textwidth}{!}{
\begin{tabular}{llcccccccccc}
\toprule
&   & \multicolumn{5}{c}{WMT19 En-De}  & \multicolumn{5}{c}{WMT19 En-Ru}  \\
\cmidrule{1-2}\cmidrule(l){3-7}\cmidrule(l){8-12}
\multicolumn{2}{c}{Num. of Samples} & 4    & 8    & 16   & 32   & 64   & 4    & 8    & 16   & 32   & 64   \\\cmidrule{1-2}\cmidrule(l){3-7}\cmidrule(l){8-12}
\multirow{4}{*}{Num. of Models}  & 1 & 1,501 & 1,501 & 1,501 & 1,501 & 1,528 & 1,533 & 1,533 & 1,533 & 1,533 & 1,566 \\\cmidrule{2-2}\cmidrule(l){3-7}\cmidrule(l){8-12}
 & 2  & 2,610 & 2,610 & 2,610 & 2,610 & 2,636 & 2,642 & 2,642 & 2,642 & 2,642 & 2,674 \\
 & 4   & 4,828 & 4828 & 4,828 & 4,828 & 4,855 & 4,859 & 4,859 & 4,859 & 4,859 & 4,891 \\
 & 8  & 9,261 & 9,261 & 9,261 & 9,261 & 9,289 & 9,293 & 9,293 & 9,293 & 9,293 & 9,326\\
\midrule
&   & \multicolumn{5}{c}{SAMSum}  & \multicolumn{5}{c}{XSum}  \\
\cmidrule{1-2}\cmidrule(l){3-7}\cmidrule(l){8-12}
\multicolumn{2}{c}{Num. of Samples} & 4    & 8    & 16   & 32   & 64   & 4    & 8    & 16   & 32   & 64   \\\cmidrule{1-2}\cmidrule(l){3-7}\cmidrule(l){8-12}
\multirow{4}{*}{Num. of Models} & 1 & 2,525 & 2,529 & 2,525 & 2,526 & 2,524 & 2,104 & 2,102 & 2,106 & 2,102 & 2,106 \\\cmidrule{2-2}\cmidrule(l){3-7}\cmidrule(l){8-12}
 & 2 & 3,734 & 3,738 & 3,734 & 3,736 & 3,734 & 3,313 & 3,311 & 3,315 & 3,311 & 3,315 \\
 & 4 & 6,152 & 6,156 & 6,153 & 6,154 & 6,152 & 5,732 & 5,730 & 5,733 & 5,730 & 5,733 \\
 & 8 & 10,990 & 10,994 & 10,990 & 10,992 & 10,990 & 10,569 & 10,567 & 10,570 & 10,567 & 10,571 \\
\midrule
&   & \multicolumn{5}{c}{MSCOCO}  & \multicolumn{5}{c}{NoCaps}  \\
\cmidrule{1-2}\cmidrule(l){3-7}\cmidrule(l){8-12}
\multicolumn{2}{c}{Num. of Samples} & 4    & 8    & 16   & 32   & 64   & 4    & 8    & 16   & 32   & 64   \\\cmidrule{1-2}\cmidrule(l){3-7}\cmidrule(l){8-12}
\multirow{4}{*}{Num. of Models}  & 1 & 1,621 & 1,621 & 1,621 & 1,621 & 1,621 & 1,654 & 1,654 & 1,655 & 1,654 & 1,655 \\\cmidrule{2-2}\cmidrule(l){3-7}\cmidrule(l){8-12}
 & 2 & 2,831 & 2,831 & 2,831 & 2,831 & 2,831 & 2,863 & 2,864 & 2,863 & 2,864 & 2,863\\
 & 4 & 5,249 & 5,249 & 5,249 & 5,249 & 5,249 & 5,282 & 5,282 & 5,282 & 5,283 & 5,282 \\
 & 8 & 10,086 & 10,086 & 10,086 & 10,086 & 10,086 & 10,119 & 10,119 & 10,119 & 10,120 & 10,119 \\
\bottomrule
\end{tabular}
}
\caption{GPU memory usages (MB) by MAMBR in each setting with ancestral sampling.
}
\label{tab:mambr_memory}
\end{table*}
\begin{table*}[t]
\small
\centering
\resizebox{0.88\textwidth}{!}{
\begin{tabular}{llcccccccccc}
\toprule
                                &   & \multicolumn{5}{c}{WMT19 En-De}  & \multicolumn{5}{c}{WMT19 En-Ru}  \\
                                \cmidrule{1-2}\cmidrule(l){3-7}\cmidrule(l){8-12}
                            \multicolumn{2}{c}{Num. of Samples} & 4    & 8    & 16   & 32   & 64   & 4    & 8    & 16   & 32   & 64   \\
                                \cmidrule{1-2}\cmidrule(l){3-7}\cmidrule(l){8-12}
\multirow{4}{*}{Num. of Models} & 1 & 85.9          & 86.1          & 86.2          & 86.2          & 86.2          & 87.3          & 87.6          & \textbf{87.7} & \textbf{87.7} & 87.7          \\\cmidrule{2-2}\cmidrule(l){3-7}\cmidrule(l){8-12}
& 2 & \textbf{86.0} & 86.1          & 86.1          & 86.2          & 86.3          & 87.3          & 87.6          & \textbf{87.7} & \textbf{87.7} & 87.7          \\
& 4 & \textbf{86.0} & 86.1          & 86.2          & 86.2          & 86.3          & \textbf{87.4} & 87.6          & \textbf{87.7} & \textbf{87.7} & 87.7          \\
& 8 & \textbf{86.0} & \textbf{86.2} & \textbf{86.3} & \textbf{86.3} & \textbf{86.4} & \textbf{87.4} & \textbf{87.7} & \textbf{87.7} & \textbf{87.7} & \textbf{87.8} \\
\midrule
                                &   & \multicolumn{5}{c}{SAMSum}  & \multicolumn{5}{c}{XSum}  \\
                                \cmidrule{1-2}\cmidrule(lr){3-7}\cmidrule(lr){8-12}
                            \multicolumn{2}{c}{Num. of Samples} & 4    & 8    & 16   & 32   & 64   & 4    & 8    & 16   & 32   & 64   \\\cmidrule{1-2}\cmidrule(lr){3-7}\cmidrule(lr){8-12}
\multirow{4}{*}{Num. of Models} & 1 & 27.5          & 27.9          & 28.3          & 28.4          & 28.5          & \textbf{54.9} & 55.7          & 56.1 & 56.3 & 56.4          \\\cmidrule{2-2}\cmidrule(l){3-7}\cmidrule(l){8-12}
& 2 & \textbf{27.7} & 28.1          & 28.5          & \textbf{28.6} & \textbf{28.7} & \textbf{54.9} & 55.7          & 56.1 & 56.3 & \textbf{56.5} \\
& 4 & \textbf{27.7} & \textbf{28.2} & 28.5          & \textbf{28.6} & 28.6          & \textbf{54.9} & 55.7          & 56.1 & \textbf{56.4} & \textbf{56.5} \\
& 8 & 27.6 & \textbf{28.2} & \textbf{28.6} & \textbf{28.6} & \textbf{28.7} & \textbf{54.9} & \textbf{55.8} & \textbf{56.2} & \textbf{56.4} & \textbf{56.5} \\
\midrule
                                &   & \multicolumn{5}{c}{MSCOCO}  & \multicolumn{5}{c}{NoCaps}  \\
                                \cmidrule{1-2}\cmidrule(lr){3-7}\cmidrule(lr){8-12}
                            \multicolumn{2}{c}{Num. of Samples} & 4    & 8    & 16   & 32   & 64   & 4    & 8    & 16   & 32   & 64   \\\cmidrule{1-2}\cmidrule(lr){3-7}\cmidrule(lr){8-12}
\multirow{4}{*}{Num. of Models}  & 1 & 55.2          & 55.9          & \textbf{56.3} & 56.5          & 56.7          & 44.4 & 46.7          & 48.5 & 49.1 & 49.5          \\\cmidrule{2-2}\cmidrule(l){3-7}\cmidrule(l){8-12}
 & 2 & 55.2 & 55.9          & \textbf{56.3} & 56.5 & 56.7 & 44.4 & 46.8          & 48.6 & 49.2 & 49.6 \\
 & 4 & 55.2 & 56.0 & \textbf{56.3} & \textbf{56.6} & \textbf{56.8} & 44.5 & 46.9          & \textbf{48.7} & 49.3 & \textbf{49.7} \\
 & 8 & \textbf{55.3} & \textbf{56.1} & \textbf{56.3} & \textbf{56.6} & \textbf{56.8} & \textbf{44.6} & \textbf{47.0} & \textbf{48.7} & \textbf{49.4} & \textbf{49.7}\\
\bottomrule
\end{tabular}
}
\caption{Results of MAMBR with epsilon sampling. Notations are the same as Table \ref{tab:mambr_ancestral_full}.
}
\label{tab:mambr_epsilon_full}
\end{table*}
\begin{table*}[t]
\small
\centering
\resizebox{0.88\textwidth}{!}{
\begin{tabular}{llcccccccccc}
\toprule
                                &   & \multicolumn{5}{c}{WMT19 En-De}  & \multicolumn{5}{c}{WMT19 En-Ru}  \\
                                \cmidrule{1-2}\cmidrule(l){3-7}\cmidrule(l){8-12}
                            \multicolumn{2}{c}{Num. of Samples} & 4    & 8    & 16   & 32   & 64   & 4    & 8    & 16   & 32   & 64   \\\cmidrule{1-2}\cmidrule(l){3-7}\cmidrule(l){8-12}
\multirow{4}{*}{Num. of Models}  & 1 & 85.2          & 85.4          & 85.6          & 85.7          & 85.8          & 86.5          & 86.8          & \textbf{87.0} & \textbf{87.1} & 87.1          \\\cmidrule{2-2}\cmidrule(l){3-7}\cmidrule(l){8-12}
 & 2 & \textbf{85.3} & \textbf{85.5} & \textbf{85.7} & \textbf{85.8} & 85.8          & 86.5          & 86.8          & 86.9          & \textbf{87.1} & 87.1          \\
 & 4 & \textbf{85.3} & \textbf{85.5} & \textbf{85.7} & \textbf{85.8} & 85.8          & \textbf{86.6} & \textbf{86.9} & \textbf{87.0} & \textbf{87.1} & \textbf{87.2} \\
 & 8 & \textbf{85.3} & \textbf{85.5} & \textbf{85.7} & \textbf{85.8} & \textbf{85.9} & 86.5          & 86.8          & \textbf{87.0} & \textbf{87.1} & \textbf{87.2}\\
\midrule
                                &   & \multicolumn{5}{c}{SAMSum}  & \multicolumn{5}{c}{XSum}  \\\cmidrule{1-2}\cmidrule(l){3-7}\cmidrule(l){8-12}
                            \multicolumn{2}{c}{Num. of Samples} & 4    & 8    & 16   & 32   & 64   & 4    & 8    & 16   & 32   & 64   \\\cmidrule{1-2}\cmidrule(l){3-7}\cmidrule(l){8-12}
\multirow{4}{*}{Num. of Models}  & 1 & 27.6          & 28.7          & \textbf{29.2} & 29.3          & \textbf{29.7} & \textbf{53.8} & \textbf{54.0} & \textbf{54.2} & \textbf{54.2} & \textbf{54.4} \\\cmidrule{2-2}\cmidrule(l){3-7}\cmidrule(l){8-12}
 & 2 & \textbf{27.8} & \textbf{28.9} & \textbf{29.2} & \textbf{29.4} & \textbf{29.7} & \textbf{53.8} & 53.9          & 54.1          & \textbf{54.2} & \textbf{54.4} \\
 & 4 & \textbf{27.8} & \textbf{28.9} & \textbf{29.2} & \textbf{29.4} & \textbf{29.7} & \textbf{53.8} & \textbf{54.0} & \textbf{54.2} & \textbf{54.2} & \textbf{54.4} \\
 & 8 & \textbf{27.8} & \textbf{28.9} & \textbf{29.2} & \textbf{29.4} & \textbf{29.7} & \textbf{53.8} & \textbf{54.0} & \textbf{54.2} & \textbf{54.2} & \textbf{54.4} \\
\midrule
                                &   & \multicolumn{5}{c}{MSCOCO}  & \multicolumn{5}{c}{NoCaps}  \\
                                \cmidrule{1-2}\cmidrule(l){3-7}\cmidrule(l){8-12}
                            \multicolumn{2}{c}{Num. of Samples} & 4    & 8    & 16   & 32   & 64   & 4    & 8    & 16   & 32   & 64   \\\cmidrule{1-2}\cmidrule(l){3-7}\cmidrule(l){8-12}
\multirow{4}{*}{Num. of Models}  & 1 & 55.4          & 55.7          & \textbf{55.9} & \textbf{56.1} & \textbf{56.3} & \textbf{48.2} & \textbf{48.8} & 49.4          & \textbf{49.9} & 50.2          \\\cmidrule{2-2}\cmidrule(l){3-7}\cmidrule(l){8-12}
 & 2 & 55.4          & \textbf{55.8} & \textbf{55.9} & \textbf{56.1} & \textbf{56.3} & \textbf{48.2} & \textbf{48.8} & 49.4          & \textbf{49.9} & \textbf{50.3} \\
 & 4 & \textbf{55.5} & 55.7          & \textbf{55.9} & \textbf{56.1} & \textbf{56.3} & \textbf{48.2} & \textbf{48.8} & \textbf{49.5} & \textbf{49.9} & 50.2          \\
 & 8 & \textbf{55.5} & 55.7          & \textbf{55.9} & \textbf{56.1} & \textbf{56.3} & \textbf{48.2} & \textbf{48.8} & \textbf{49.5} & \textbf{49.9} & 50.2      \\
\bottomrule
\end{tabular}
}
\caption{Results of MAMBR with beam decoding. Notations are the same as Table \ref{tab:mambr_ancestral_full}.}
\label{tab:mambr_beam_full}
\end{table*}

\section{Detailed Results of MAMBR}
\label{appendix:bias-diversity-mambr}

\subsection{Computational Costs}

The computational costs of MAMBR are proportional to the number of models used. Tables \ref{tab:mambr_time} and \ref{tab:mambr_memory} show the computational cost of running the settings corresponding to Table \ref{tab:mambr_ancestral_full}. We used one NVIDIA RTX A6000 for the measurement.

\subsection{Results on other Sampling Strategies}

Tables \ref{tab:mambr_epsilon_full} and \ref{tab:mambr_beam_full} show the results of MAMBR with epsilon sampling and beam decoding, respectively. In the following discussion, we also consider Table \ref{tab:mambr_ancestral_full}. In ancestral and epsilon sampling, the best and moderately diversified sampling strategies (as shown in Figure \ref{fig:bias_diversity:full}), we observe performance improvement as the number of models increases. On the other hand, in the lowest diversity method, beam decoding, performance improvement is limited. These results suggest that MAMBR can improve performance by enhancing the diversity of evaluation metrics, although the diversity of the sampling strategy itself remains important.

\subsection{Bias and Diversity of MAMBR}
Figures \ref{fig:mambr:bias_diversity:full:ancestral}, \ref{fig:mambr:bias_diversity:full:epsilon}, and \ref{fig:mambr:bias_diversity:full:beam} show the bias and diversity corresponding to the results in Tables \ref{tab:mambr_ancestral_full}, \ref{tab:mambr_epsilon_full}, and \ref{tab:mambr_beam_full}, respectively. The results show that MAMBR actually increases the diversity in WMT19 En-De and En-Ru and SAMSum but not in the other datasets. Thus, this improvement depends on the datasets. On the other hand, we can see the improvement of bias in some cases. This is reasonable because using multiple metric models itself is an ensembling approach and can contribute to performance improvement.

\section{Experimental Results on the First 1000 Examples}

To consider more detailed configurations and reveal the possibility of a more efficient investigation, we conducted an additional evaluation using only the first 1000 examples for each dataset based on the setting of \citet{pmlr-v235-jinnai24a}.

\subsection{Hypotheses generated by different sampling strategies}
\label{app:results:hyp:sampling}

Figures \ref{fig:bias_diversity:1000:beam} to \ref{fig:bias_diversity:1000:epsilon} present the bias and diversity decomposition plots for different hypothesis generation strategies. The results indicate that differences in the generated hypotheses influence performance in some cases, whereas the overall tendencies of the sampling strategy used for generating pseudo-references remain similar despite these variations.

\subsection{MAMBR}

Tables \ref{tab:mambr_ancestral_sampling:1000} to \ref{tab:mambr_beam_decoding:1000} show the MAMBR results for the first 1000 lines. From these results, we observe a similar trend to those obtained when the dataset is fully used, as described in \S\ref{subsec:experiments:mambr}. Similarly, Figures \ref{fig:bias_diversity:1000:mambr:ancestral} to \ref{fig:bias_diversity:1000:mambr:beam} demonstrate that the results are nearly identical to those obtained when the dataset is fully utilized.

\section{Reproducibility Statement}

We performed our experiments by running publicly available models, facebook/wmt19-en-de (Apache license 2.0), facebook/wmt19-en-ru (Apache license 2.0), philschmid/bart-large-cnn-samsum (MIT License), facebook/bart-large-xsum (MIT License), Salesforce/blip2-flan-t5-xl-coco (MIT License), and Salesforce/blip2-flan-t5-xl (MIT License) in HuggingFace Transformers \citep{wolf-etal-2020-transformers} on the publicly available datasets, WMT19 English to German \citep{barrault-etal-2019-findings}, WMT19 English to Russian \citep{barrault-etal-2019-findings}, SAMSum \citep{gliwa-etal-2019-samsum}, XSum \citep{narayan-etal-2018-dont}, MSCOCO \citep{Lin2014MicrosoftCC,karpathy2015deep}, and NoCaps \citep{agrawal2019nocaps}, respectively with utilizing the publicly available MBR decoding toolkit, \texttt{mbrs} \citep{deguchi-2024-mbrs} as described in  \S\ref{subsec:overall:settings}.

\begin{figure*}
    \centering
    \includegraphics[width=0.975\textwidth]{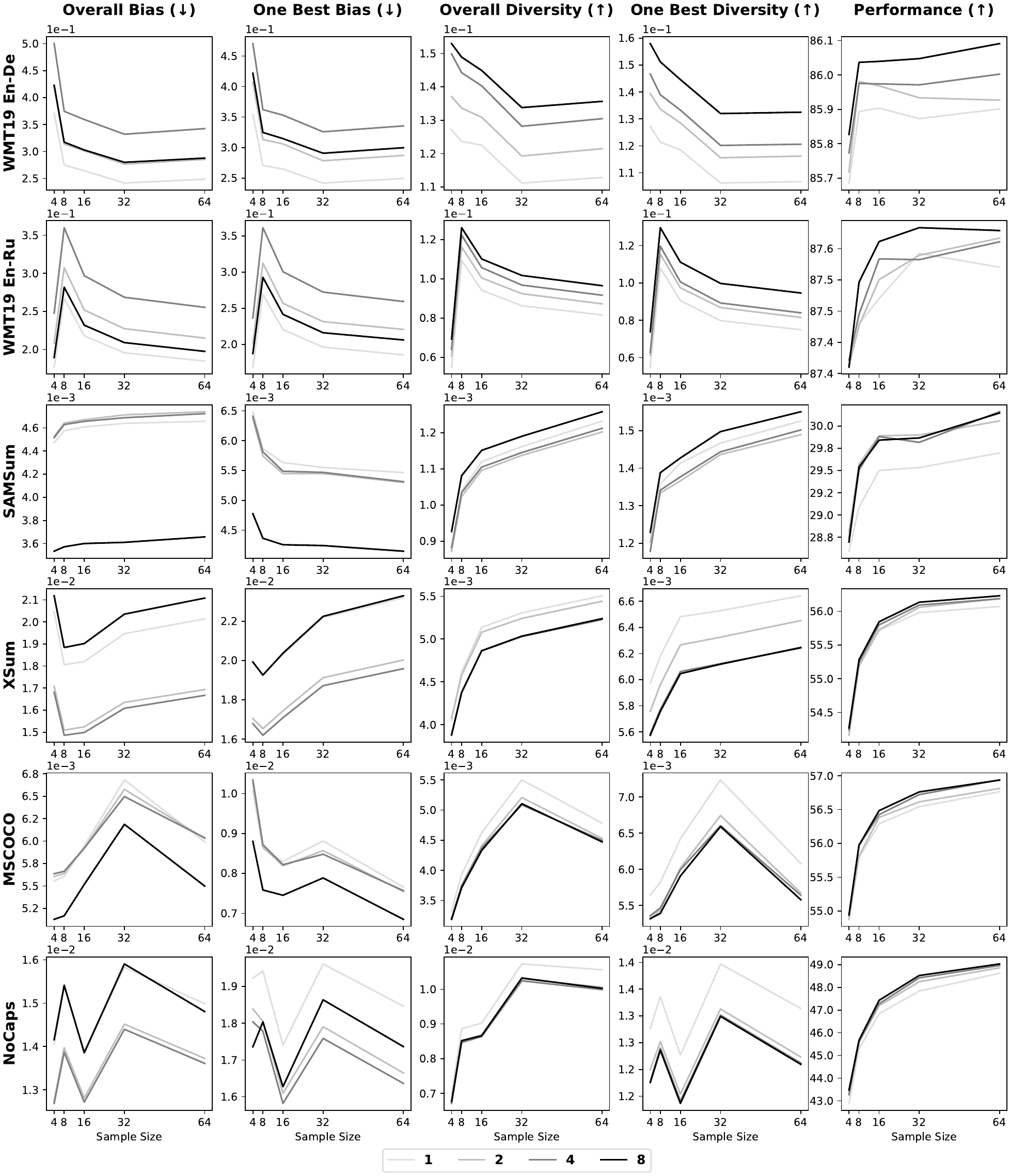}
    \caption{The relationship between bias, diversity, and performance in MAMBR decoding with pseudo-references generated by ancestral sampling. The lines indicate the score for each number of metric models used. Other notations are the same as Figure \ref{fig:bias_diversity:full}.}
    \label{fig:mambr:bias_diversity:full:ancestral}
\end{figure*}
\begin{figure*}
    \centering
    \includegraphics[width=0.975\textwidth]{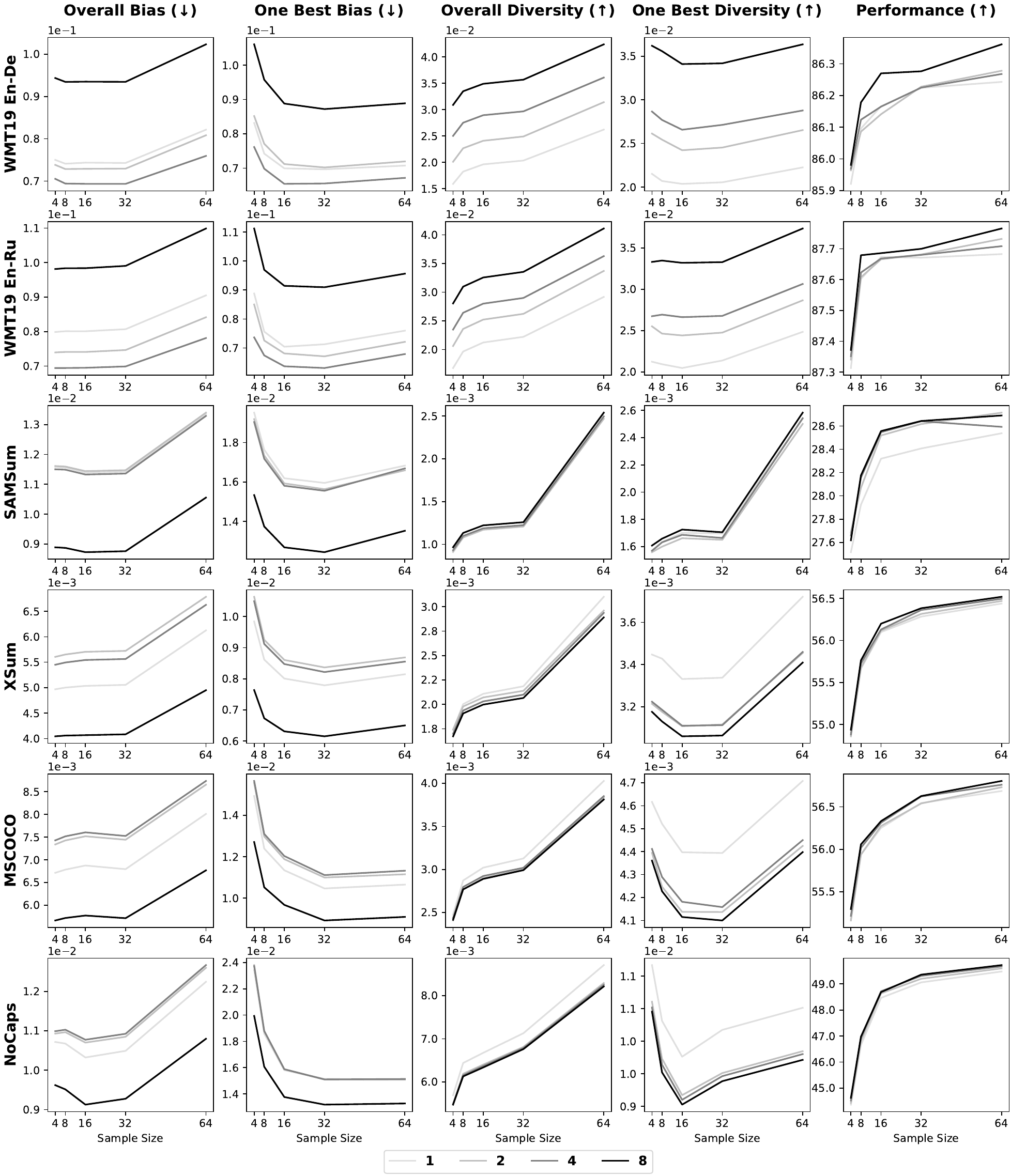}
    \caption{The relationship between bias, diversity, and performance in MAMBR decoding with pseudo-references generated by epsilon sampling. The notations are the same as Figure \ref{fig:mambr:bias_diversity:full:ancestral}.}
    \label{fig:mambr:bias_diversity:full:epsilon}
\end{figure*}
\begin{figure*}
    \centering
    \includegraphics[width=0.975\textwidth]{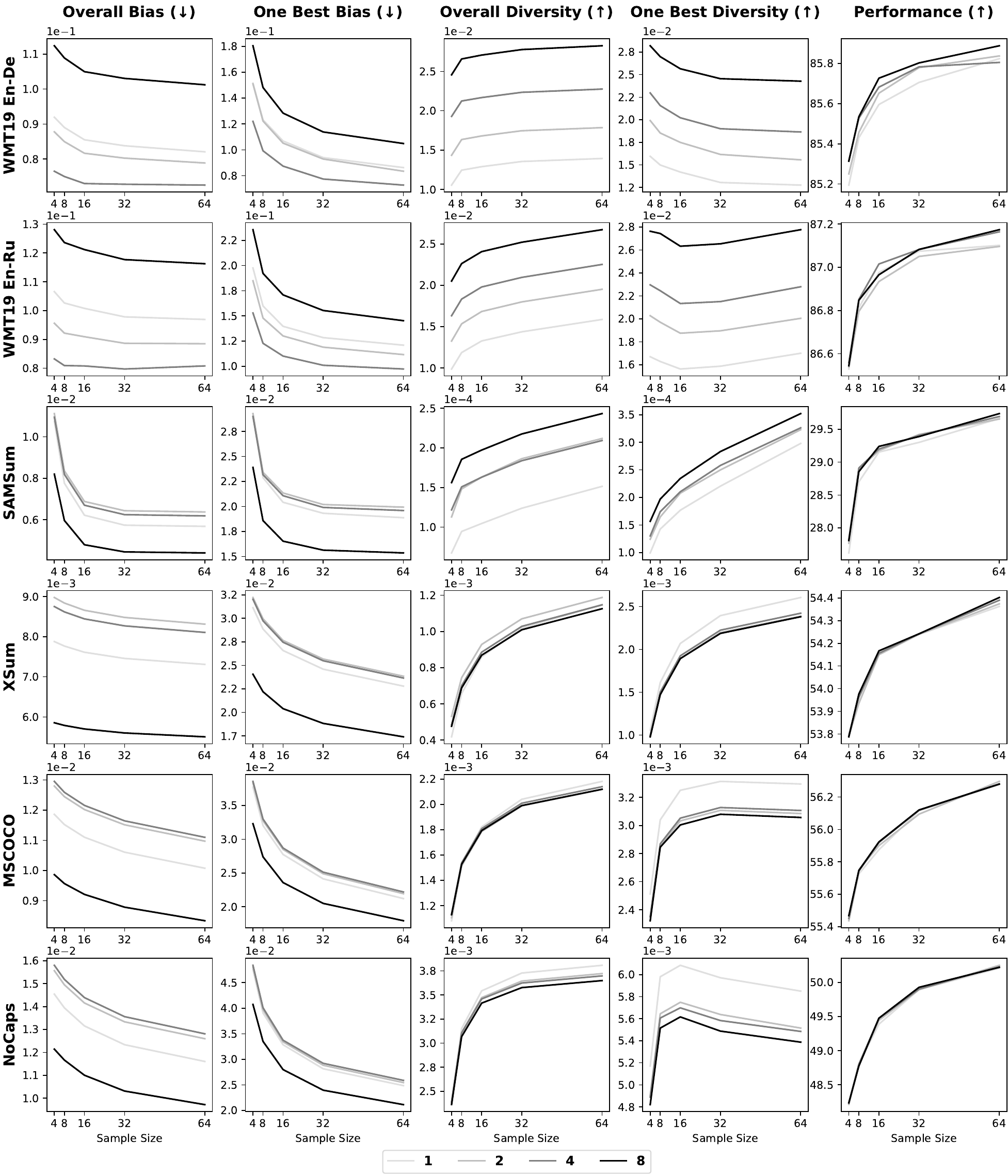}
    \caption{The relationship between bias, diversity, and performance in MAMBR decoding with pseudo-references generated by beam decoding. The notations are the same as Figure \ref{fig:mambr:bias_diversity:full:ancestral}.}
    \label{fig:mambr:bias_diversity:full:beam}
\end{figure*}

\begin{figure*}[t]
    \centering
    \includegraphics[width=0.975\textwidth]{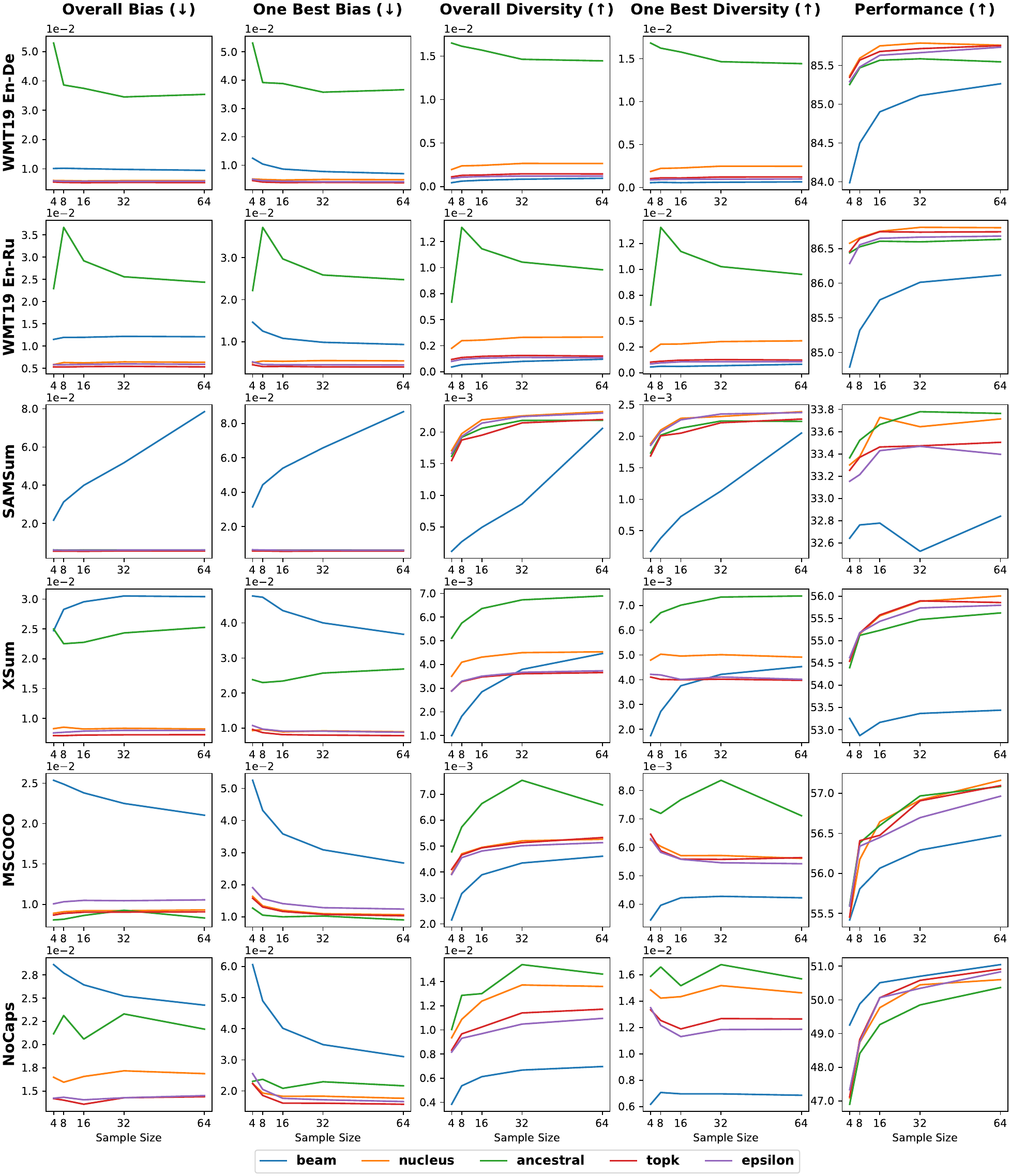}
    \caption{The relationship between bias, diversity, and performance on the first 1000 lines of each dataset in MBR decoding with hypotheses generated by beam decoding. The notations are the same as Figure \ref{fig:bias_diversity:full}.}
    \label{fig:bias_diversity:1000:beam}
\end{figure*}

\begin{figure*}[t]
    \centering
    \includegraphics[width=0.975\textwidth]{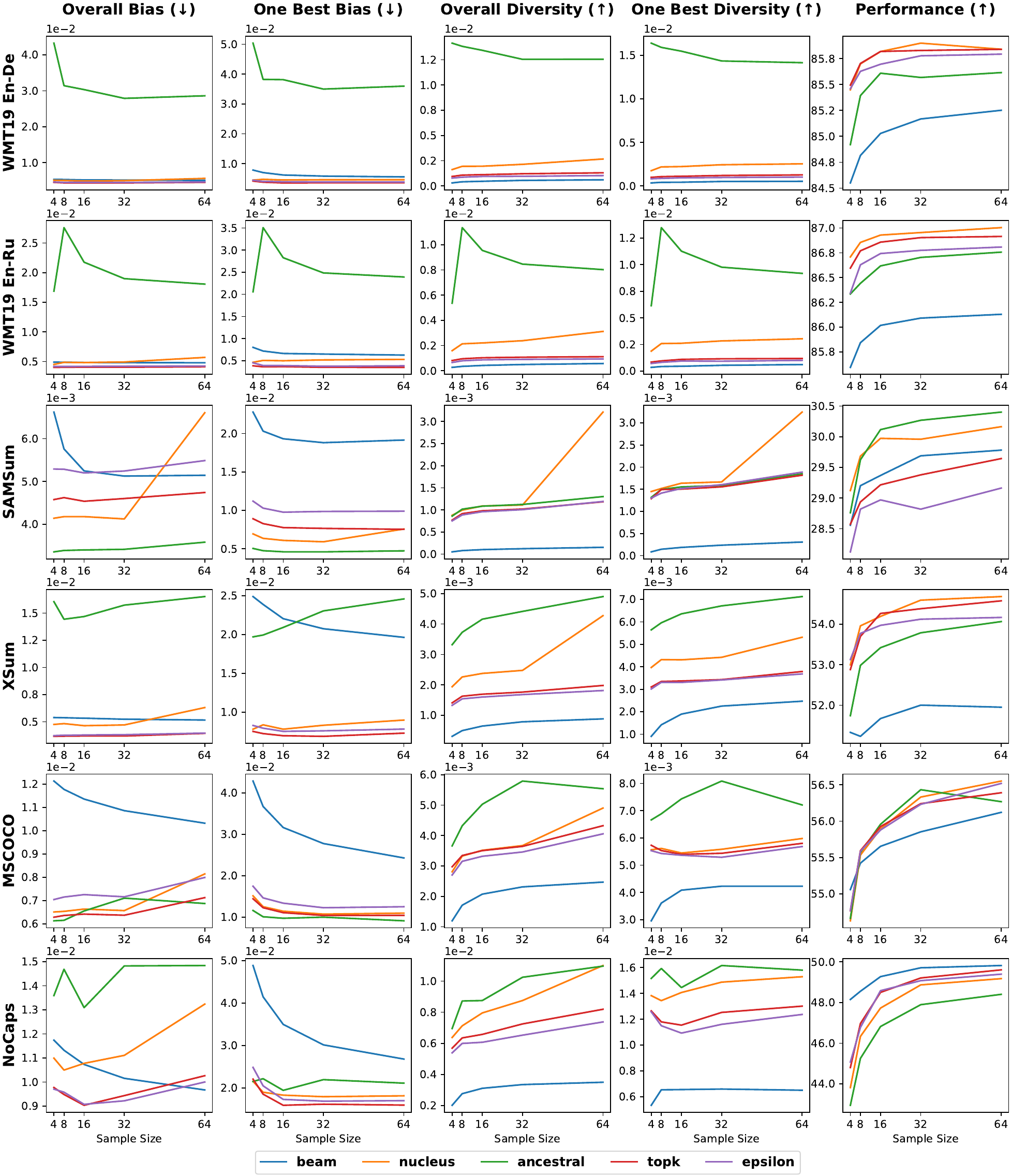}
    \caption{The relationship between bias, diversity, and performance on the first 1000 lines of each dataset in MBR decoding with hypotheses generated by nucleus sampling. The notations are the same as Figure \ref{fig:bias_diversity:full}.}
    \label{fig:bias_diversity:1000:nucleus}
\end{figure*}

\begin{figure*}[t]
    \centering
    \includegraphics[width=0.975\textwidth]{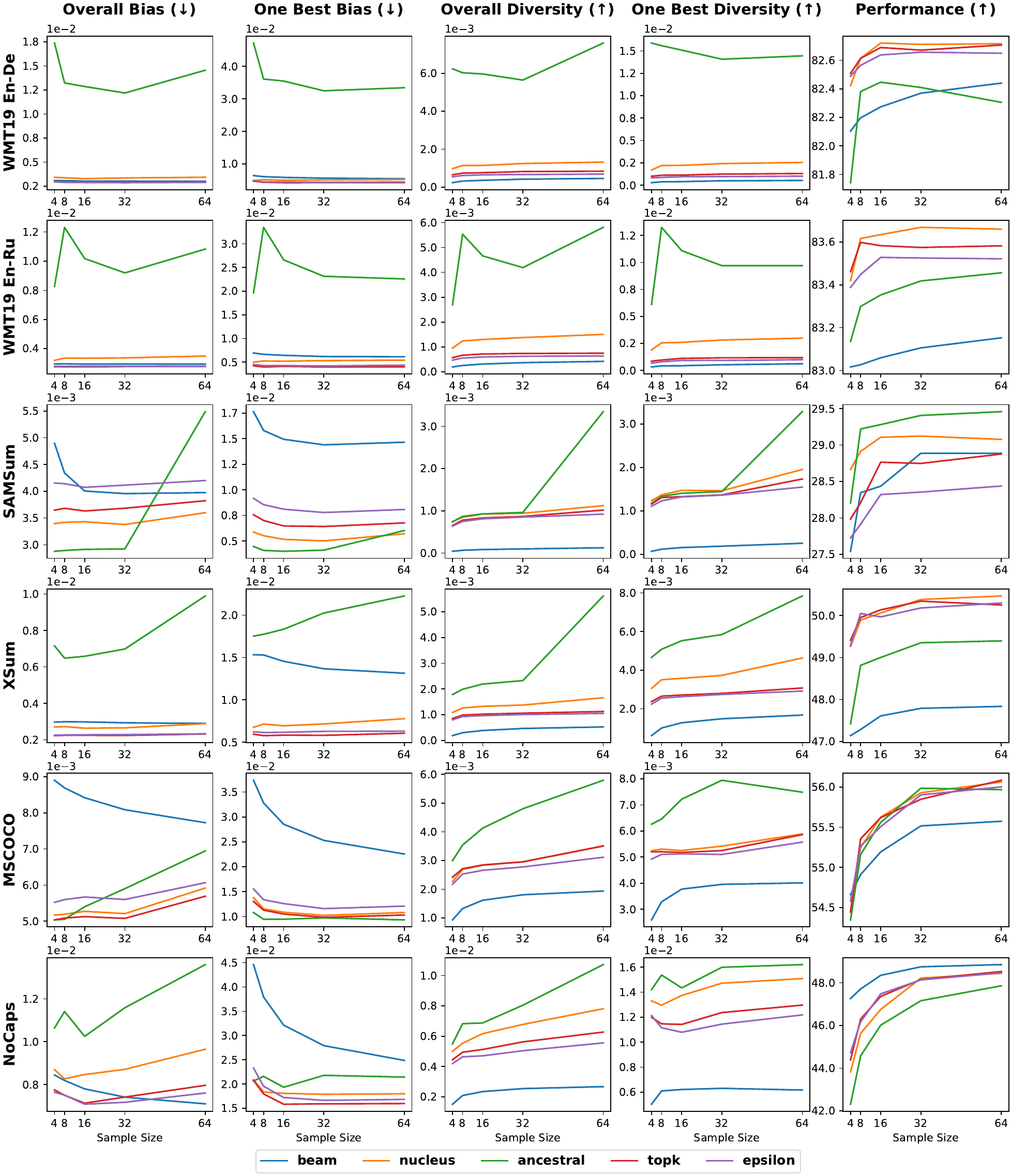}
    \caption{The relationship between bias, diversity, and performance on the first 1000 lines of each dataset in MBR decoding with hypotheses generated by ancestral sampling. The notations are the same as Figure \ref{fig:bias_diversity:full}.}
    \label{fig:bias_diversity:1000:ancestral}
\end{figure*}

\begin{figure*}[t]
    \centering
    \includegraphics[width=0.975\textwidth]{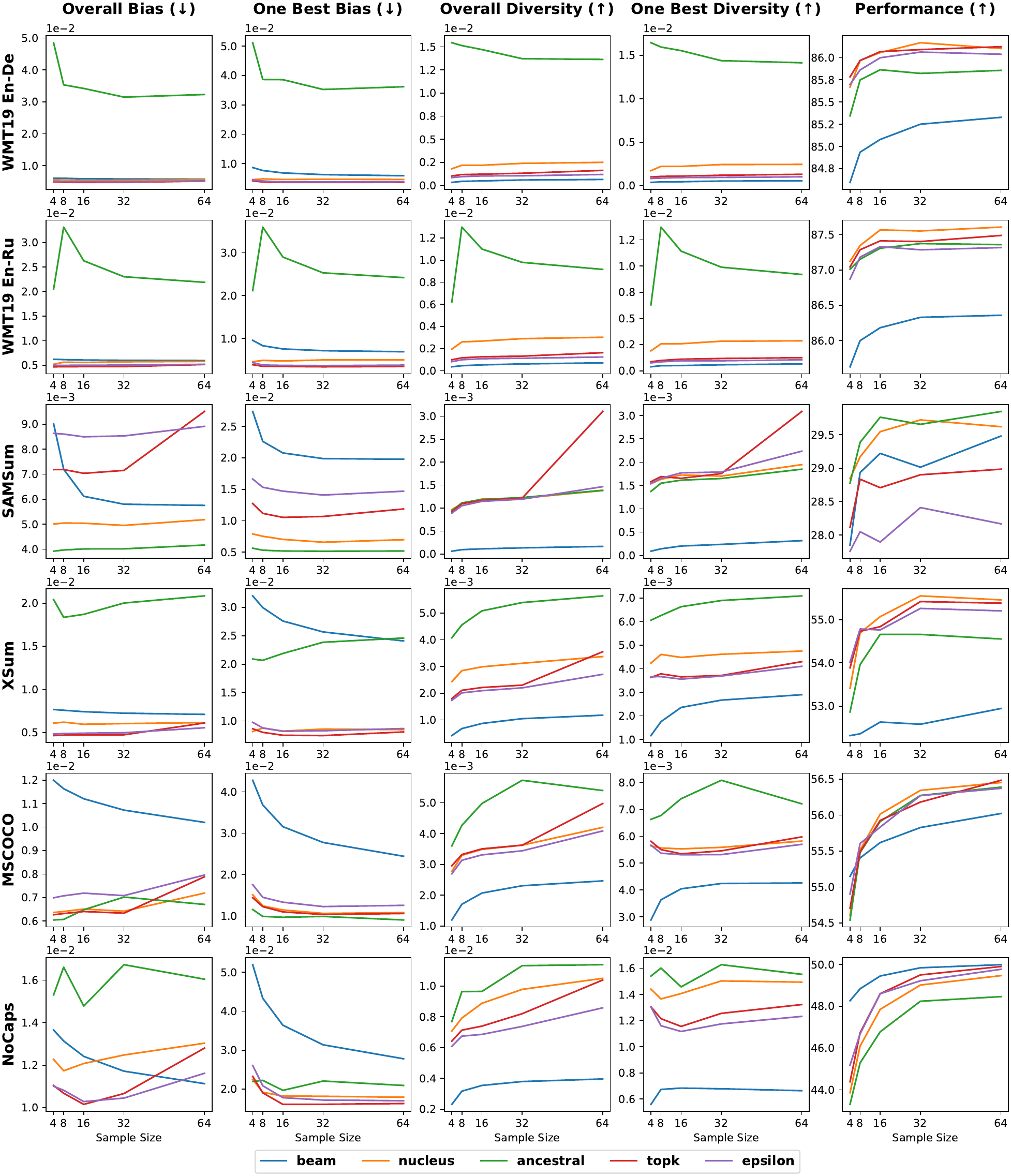}
    \caption{The relationship between bias, diversity, and performance on the first 1000 lines of each dataset in MBR decoding with hypotheses generated by top-k sampling. The notations are the same as Figure \ref{fig:bias_diversity:full}.}
    \label{fig:bias_diversity:1000:topk}
\end{figure*}

\begin{figure*}[t]
    \centering
    \includegraphics[width=0.975\textwidth]{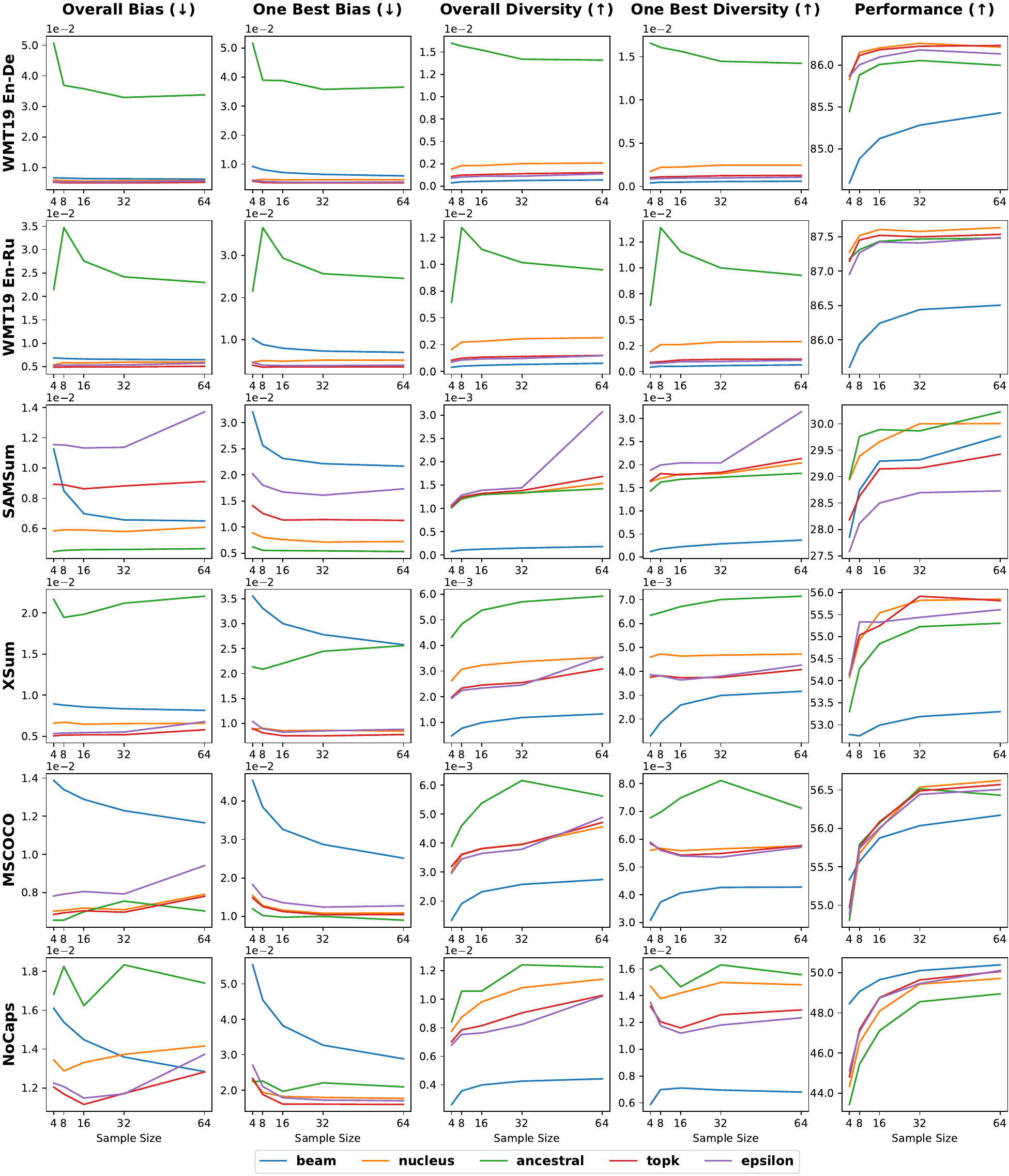}
    \caption{The relationship between bias, diversity, and performance on the first 1000 lines of each dataset in MBR decoding with hypotheses generated by epsilon sampling. The notations are the same as Figure \ref{fig:bias_diversity:full}.}
    \label{fig:bias_diversity:1000:epsilon}
\end{figure*}

\begin{table}[t]
\small
\centering
\resizebox{0.88\textwidth}{!}{
\begin{tabular}{llcccccccccc}
\toprule
                                &   & \multicolumn{5}{c}{WMT19 En-De}  & \multicolumn{5}{c}{WMT19 En-Ru}  \\
                                \cmidrule{1-2}\cmidrule(l){3-7}\cmidrule(l){8-12}
                            \multicolumn{2}{c}{Num. of Samples} & 4    & 8    & 16   & 32   & 64   & 4    & 8    & 16   & 32   & 64   \\\cmidrule{1-2}\cmidrule(l){3-7}\cmidrule(l){8-12}
\multirow{4}{*}{Num. of Models} & 1 & 84.5 & 84.7 & 84.8 & 85.0 & 85.1 & \textbf{85.8} & \textbf{86.1} & 86.3 & \textbf{86.5} & 86.5 \\\cmidrule{2-2}\cmidrule(l){3-7}\cmidrule(l){8-12}
    & 2 & 84.5 & \textbf{84.8} & \textbf{85.0} & \textbf{85.1} & 85.2 & \textbf{85.8} & \textbf{86.1} & 86.3 & 86.4 & 86.5 \\
    & 4 & 84.6 & \textbf{84.8} & \textbf{85.0} & \textbf{85.1} & 85.2 & \textbf{85.8} & \textbf{86.1} & \textbf{86.4} & \textbf{86.5} & \textbf{86.6} \\
    & 8 & \textbf{84.7} & \textbf{84.8} & \textbf{85.0} & \textbf{85.1} & \textbf{85.3} & \textbf{85.8} & \textbf{86.1} & 86.3 & \textbf{86.5} & \textbf{86.6} \\
\midrule
                                &   & \multicolumn{5}{c}{SAMSum}  & \multicolumn{5}{c}{XSum}  \\
                                \cmidrule{1-2}\cmidrule(l){3-7}\cmidrule(l){8-12}
                            \multicolumn{2}{c}{Num. of Samples} & 4    & 8    & 16   & 32   & 64   & 4    & 8    & 16   & 32   & 64   \\\cmidrule{1-2}\cmidrule(l){3-7}\cmidrule(l){8-12}
\multirow{4}{*}{Num. of Models} & 1 & 28.6 & 29.1 & 29.5 & 29.5 & 29.7 & 53.3 & 54.1 & \textbf{55.0} & 55.1 & 55.2 \\\cmidrule{2-2}\cmidrule(l){3-7}\cmidrule(l){8-12}
    & 2 & \textbf{28.8} & \textbf{29.6} & \textbf{29.9} & \textbf{29.9} & 30.1 & 53.2 & 54.2 & \textbf{55.0} & \textbf{55.3} & \textbf{55.3} \\
    & 4 & 28.7 & 29.5 & \textbf{29.9} & 29.8 & \textbf{30.2} & 53.3 & 54.2 & \textbf{55.0} & \textbf{55.3} & \textbf{55.3} \\
    & 8 & 28.7 & 29.5 & 29.8 & \textbf{29.9} & 30.1 & \textbf{53.4} & \textbf{54.3} & \textbf{55.0} & 55.2 & \textbf{55.3} \\
\midrule
                                &   & \multicolumn{5}{c}{MSCOCO}  & \multicolumn{5}{c}{NoCaps}  \\
                                \cmidrule{1-2}\cmidrule(l){3-7}\cmidrule(l){8-12}
                            \multicolumn{2}{c}{Num. of Samples} & 4    & 8    & 16   & 32   & 64   & 4    & 8    & 16   & 32   & 64   \\\cmidrule{1-2}\cmidrule(l){3-7}\cmidrule(l){8-12}
\multirow{4}{*}{Num. of Models} & 1 & 54.8 & 55.8 & 56.1 & 56.4 & 56.6 & 43.2 & 45.3 & 46.9 & 48.2 & 49.0 \\\cmidrule{2-2}\cmidrule(l){3-7}\cmidrule(l){8-12}
    & 2 & 54.8 & 55.8 & 56.2 & 56.4 & 56.6 & 43.4 & 45.7 & 47.2 & 48.7 & 49.1 \\
    & 4 & 54.8 & \textbf{55.9} & \textbf{56.4} & 56.5 & \textbf{56.8} & 43.8 & 45.8 & \textbf{47.5} & 48.8 & \textbf{49.5} \\
    & 8 & \textbf{54.9} & \textbf{55.9} & 56.3 & \textbf{56.6} & \textbf{56.8} & \textbf{43.9} & \textbf{45.9} & 47.4 & \textbf{49.0} & \textbf{49.5} \\
\bottomrule
\end{tabular}
}
\caption{Results of MAMBR with samples generated by ancestral sampling. The notations are the same as Table \ref{tab:mambr_ancestral_full}.\label{tab:mambr_ancestral_sampling:1000}}
\end{table}

\begin{table}[t]
\small
\centering
\resizebox{0.88\textwidth}{!}{
\begin{tabular}{llcccccccccc}
\toprule
                                &   & \multicolumn{5}{c}{WMT19 En-De}  & \multicolumn{5}{c}{WMT19 En-Ru}  \\
                                \cmidrule{1-2}\cmidrule(l){3-7}\cmidrule(l){8-12}
                            \multicolumn{2}{c}{Num. of Samples} & 4    & 8    & 16   & 32   & 64   & 4    & 8    & 16   & 32   & 64   \\
                                \cmidrule{1-2}\cmidrule(l){3-7}\cmidrule(l){8-12}
\multirow{4}{*}{Num. of Models} & 1 & 85.2 & 85.4 & 85.6 & 85.6 & 85.6 & \textbf{86.7} & 87.0 & \textbf{87.1} & 87.0   & 87.1 \\
\cmidrule{2-2}\cmidrule(l){3-7}\cmidrule(l){8-12}
                                & 2 & \textbf{85.3} & 85.5 & 85.6 & \textbf{85.7} & 85.7 & \textbf{86.7} & 87.0 & \textbf{87.1} & \textbf{87.1} & 87.1 \\
                                & 4 & \textbf{85.3} & 85.5 & 85.6 & \textbf{85.7} & 85.7 & \textbf{86.7} & 87.0   & \textbf{87.1} & 87.0 & 87.1 \\
                                & 8 & \textbf{85.3} & \textbf{85.6} & \textbf{85.7} & \textbf{85.7} & \textbf{85.8} & \textbf{86.7} & \textbf{87.1} & \textbf{87.1} & 87.0 & \textbf{87.2} \\
\midrule
                                &   & \multicolumn{5}{c}{SAMSum}  & \multicolumn{5}{c}{XSum}  \\
                                \cmidrule{1-2}\cmidrule(lr){3-7}\cmidrule(lr){8-12}
                            \multicolumn{2}{c}{Num. of Samples} & 4    & 8    & 16   & 32   & 64   & 4    & 8    & 16   & 32   & 64   \\\cmidrule{1-2}\cmidrule(lr){3-7}\cmidrule(lr){8-12}
\multirow{4}{*}{Num. of Models} & 1 & 27.5 & 27.9 & 28.3 & 28.4 & 28.5 & 54.1 & 55.1 & 55.1 & 55.4 & 55.5 \\ \cmidrule{2-2}\cmidrule(l){3-7}\cmidrule(l){8-12}
                                & 2 & \textbf{27.7} & 28.1 & 28.5 & \textbf{28.6} & \textbf{28.7} & 54.1 & \textbf{55.2} & 55.1 & 55.4 & 55.4 \\
                                & 4 & \textbf{27.7} & \textbf{28.2} & 28.5 & \textbf{28.6} & \textbf{28.6} & 54.1 & \textbf{55.2} & 55.2 & \textbf{55.5} & \textbf{55.6} \\
                                & 8 & 27.6 & \textbf{28.2} & \textbf{28.6} & \textbf{28.6} & \textbf{28.7} & \textbf{54.2} & \textbf{55.2} & \textbf{55.3} & 55.4 & \textbf{55.6} \\
\midrule
                                &   & \multicolumn{5}{c}{MSCOCO}  & \multicolumn{5}{c}{NoCaps}  \\
                                \cmidrule{1-2}\cmidrule(lr){3-7}\cmidrule(lr){8-12}
                            \multicolumn{2}{c}{Num. of Samples} & 4    & 8    & 16   & 32   & 64   & 4    & 8    & 16   & 32   & 64   \\\cmidrule{1-2}\cmidrule(lr){3-7}\cmidrule(lr){8-12}
\multirow{4}{*}{Num. of Models} & 1 & 55.0 & 55.8 & 56.0 & 56.4 & 56.6 & 45.0 & 46.9 & 48.6 & 49.5 & 49.8 \\ \cmidrule{2-2}\cmidrule(l){3-7}\cmidrule(l){8-12}
                                & 2 & 55.0 & 55.7 & 56.1 & 56.4 & 56.6 & 45.0 & 47.1 & 48.9 & 49.5 & 50.1 \\
                                & 4 & 55.0 & \textbf{55.9} & \textbf{56.2} & \textbf{56.6} & 56.7 & \textbf{45.2} & \textbf{47.2} & 48.9 & \textbf{49.8} & 50.2 \\
                                & 8 & \textbf{55.1} & 55.8 & \textbf{56.2} & 56.5 & \textbf{56.8} & \textbf{45.2} & \textbf{47.2} & \textbf{49.0} & \textbf{49.8} & \textbf{50.3}\\
\bottomrule
\end{tabular}
}
\caption{Results of MAMBR with samples generated by epsilon sampling. The notations are the same as Table \ref{tab:mambr_ancestral_full}.\label{tab:mambr_epsilon_sampling:1000}}
\end{table}

\begin{table}[t]
\small
\centering
\resizebox{0.88\textwidth}{!}{
\begin{tabular}{llcccccccccc}
\toprule
                                &   & \multicolumn{5}{c}{WMT19 En-De}  & \multicolumn{5}{c}{WMT19 En-Ru}  \\
                                \cmidrule{1-2}\cmidrule(l){3-7}\cmidrule(l){8-12}
                            \multicolumn{2}{c}{Num. of Samples} & 4    & 8    & 16   & 32   & 64   & 4    & 8    & 16   & 32   & 64   \\\cmidrule{1-2}\cmidrule(l){3-7}\cmidrule(l){8-12}
\multirow{4}{*}{Num. of Models} & 1 & 85.1 & 85.3 & 85.4 & 85.4 & 85.3 & \textbf{86.7} & \textbf{86.9} & \textbf{86.9} & \textbf{86.9} & 86.9 \\\cmidrule{2-2}\cmidrule(l){3-7}\cmidrule(l){8-12}
& 2 & 85.1 & 85.4 & \textbf{85.5} & 85.4 & 85.4 & \textbf{86.7} & 86.8 & \textbf{86.9} & \textbf{86.9} & \textbf{87.0} \\
& 4 & \textbf{85.2} & 85.4 & \textbf{85.5} & 85.4 & 85.4 & \textbf{86.7} & 86.8 & \textbf{86.9} & \textbf{86.9} & 86.9 \\
& 8 & \textbf{85.2} & \textbf{85.5} & \textbf{85.5} & \textbf{85.5} & \textbf{85.5} & 86.6 & \textbf{86.9} & \textbf{86.9} & \textbf{86.9} & 86.9\\
\midrule
                                &   & \multicolumn{5}{c}{SAMSum}  & \multicolumn{5}{c}{XSum}  \\\cmidrule{1-2}\cmidrule(l){3-7}\cmidrule(l){8-12}
                            \multicolumn{2}{c}{Num. of Samples} & 4    & 8    & 16   & 32   & 64   & 4    & 8    & 16   & 32   & 64   \\\cmidrule{1-2}\cmidrule(l){3-7}\cmidrule(l){8-12}
\multirow{4}{*}{Num. of Models} & 1 & 27.6 & 28.7 & \textbf{29.2} & 29.3 & \textbf{29.7} & \textbf{52.8} & 52.8 & \textbf{53.1} & \textbf{53.2} & 53.3 \\\cmidrule{2-2}\cmidrule(l){3-7}\cmidrule(l){8-12}
    & 2 & \textbf{27.8} & \textbf{28.9} & \textbf{29.2} & \textbf{29.4} & \textbf{29.7} & 52.7 & \textbf{52.9} & 53.0 & \textbf{53.2} & \textbf{53.4} \\
    & 4 & \textbf{27.8} & \textbf{28.9} & \textbf{29.2} & \textbf{29.4} & \textbf{29.7} & \textbf{52.8} & 52.8 & 53.0 & \textbf{53.2} & 53.3 \\
    & 8 & \textbf{27.8} & \textbf{28.9} & \textbf{29.2} & \textbf{29.4} & \textbf{29.7} & 52.7 & \textbf{52.9} & \textbf{53.1} & \textbf{53.2} & 53.3 \\
\midrule
                                &   & \multicolumn{5}{c}{MSCOCO}  & \multicolumn{5}{c}{NoCaps}  \\
                                \cmidrule{1-2}\cmidrule(l){3-7}\cmidrule(l){8-12}
                            \multicolumn{2}{c}{Num. of Samples} & 4    & 8    & 16   & 32   & 64   & 4    & 8    & 16   & 32   & 64   \\\cmidrule{1-2}\cmidrule(l){3-7}\cmidrule(l){8-12}
\multirow{4}{*}{Num. of Models} & 1 & \textbf{55.3} & \textbf{55.5} & \textbf{55.9} & 56.0 & 56.2 & \textbf{48.5} & 49.0 & 49.6 & 50.1 & \textbf{50.5} \\\cmidrule{2-2}\cmidrule(l){3-7}\cmidrule(l){8-12}
    & 2 & \textbf{55.3} & \textbf{55.5} & 55.8 & \textbf{56.1} & 56.2 & \textbf{48.5} & 49.0 & \textbf{49.8} & 50.2 & \textbf{50.5} \\
    & 4 & \textbf{55.3} & \textbf{55.5} & 55.8 & 56.0 & \textbf{56.3} & \textbf{48.5} & 49.1 & 49.7 & 50.2 & \textbf{50.5} \\
    & 8 & \textbf{55.3} & \textbf{55.5} & 55.8 & \textbf{56.1} & 56.2 & \textbf{48.5} & \textbf{49.2} & 49.7 & \textbf{50.3} & \textbf{50.5} \\
\bottomrule
\end{tabular}
}
\caption{Results of MAMBR with samples generated by beam decoding. The notations are the same as Table \ref{tab:mambr_ancestral_full}.\label{tab:mambr_beam_decoding:1000}}
\end{table}

\begin{figure*}[t]
    \centering
    \includegraphics[width=\textwidth]{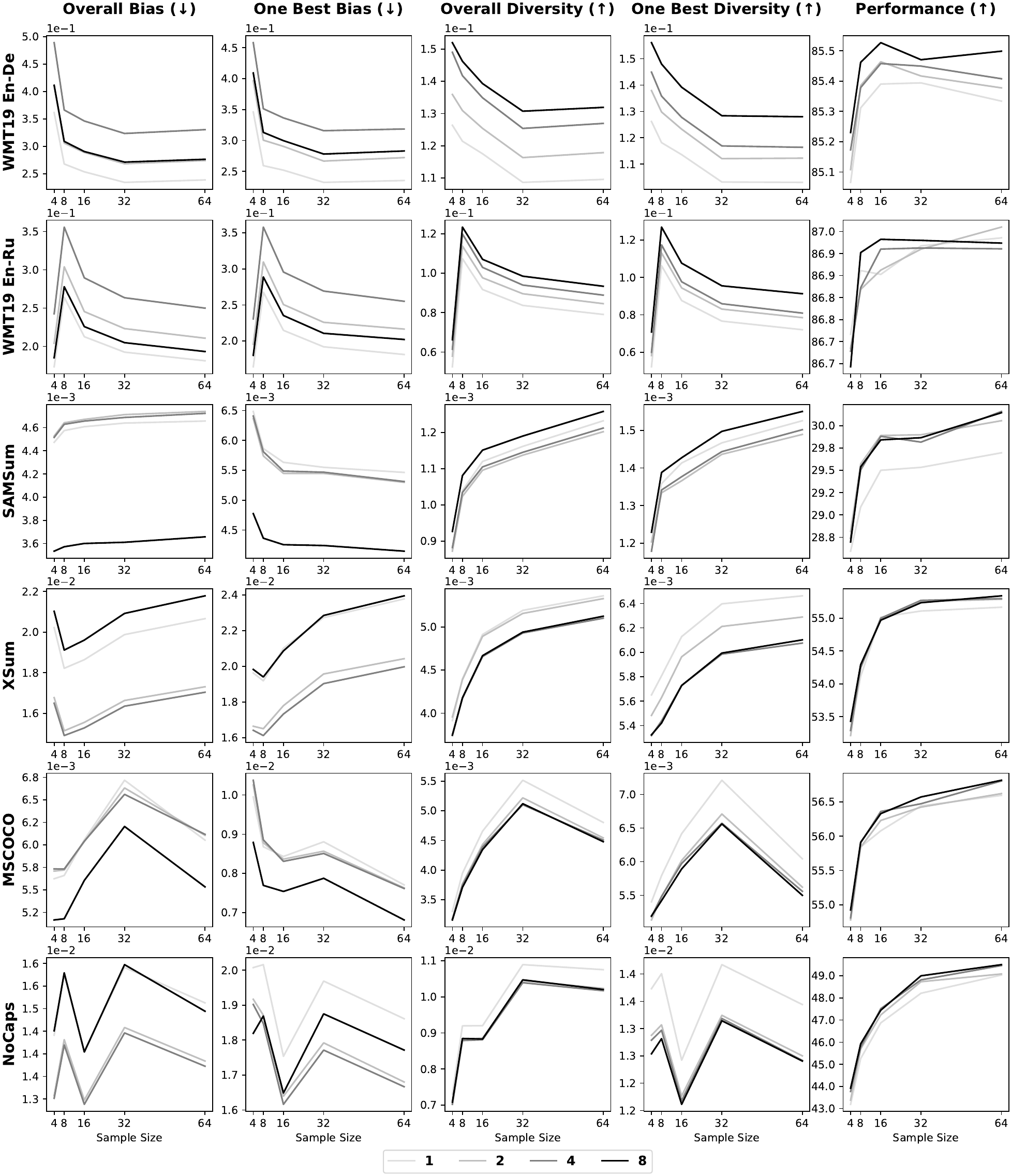}
    \caption{The relationship between bias, diversity, and performance on the first 1000 lines of each dataset in MBR decoding with pseudo-references generated by ancestral sampling. The notations are the same as Figure \ref{fig:mambr:bias_diversity:full:ancestral}.}
    \label{fig:bias_diversity:1000:mambr:ancestral}
\end{figure*}

\begin{figure*}[t]
    \centering
    \includegraphics[width=\textwidth]{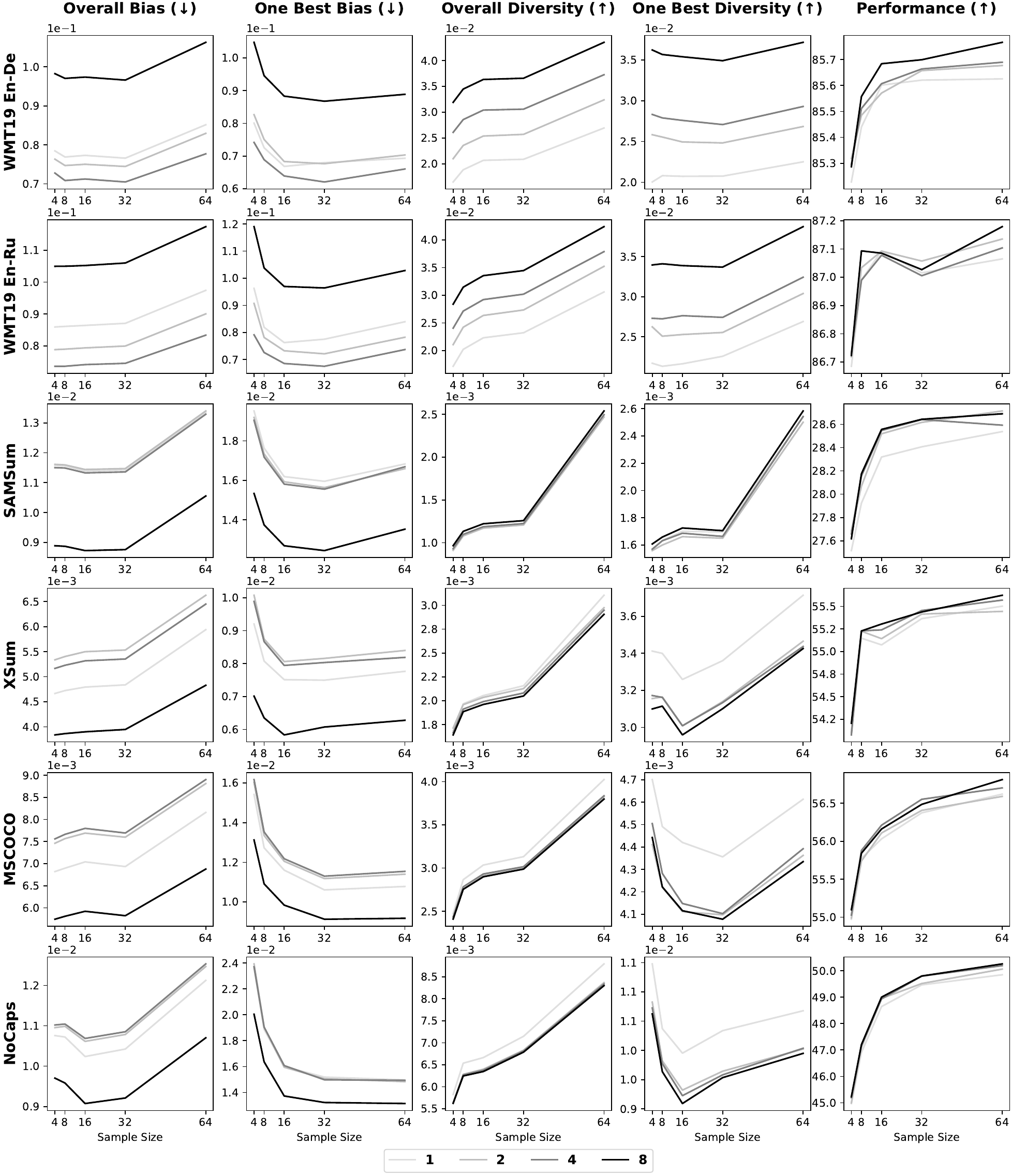}
    \caption{The relationship between bias, diversity, and performance on the first 1000 lines of each dataset in MBR decoding with pseudo-references generated by epsilon sampling. The notations are the same as Figure \ref{fig:mambr:bias_diversity:full:ancestral}.}
    \label{fig:bias_diversity:1000:mambr:epsilon}
\end{figure*}

\begin{figure*}[t]
    \centering
    \includegraphics[width=\textwidth]{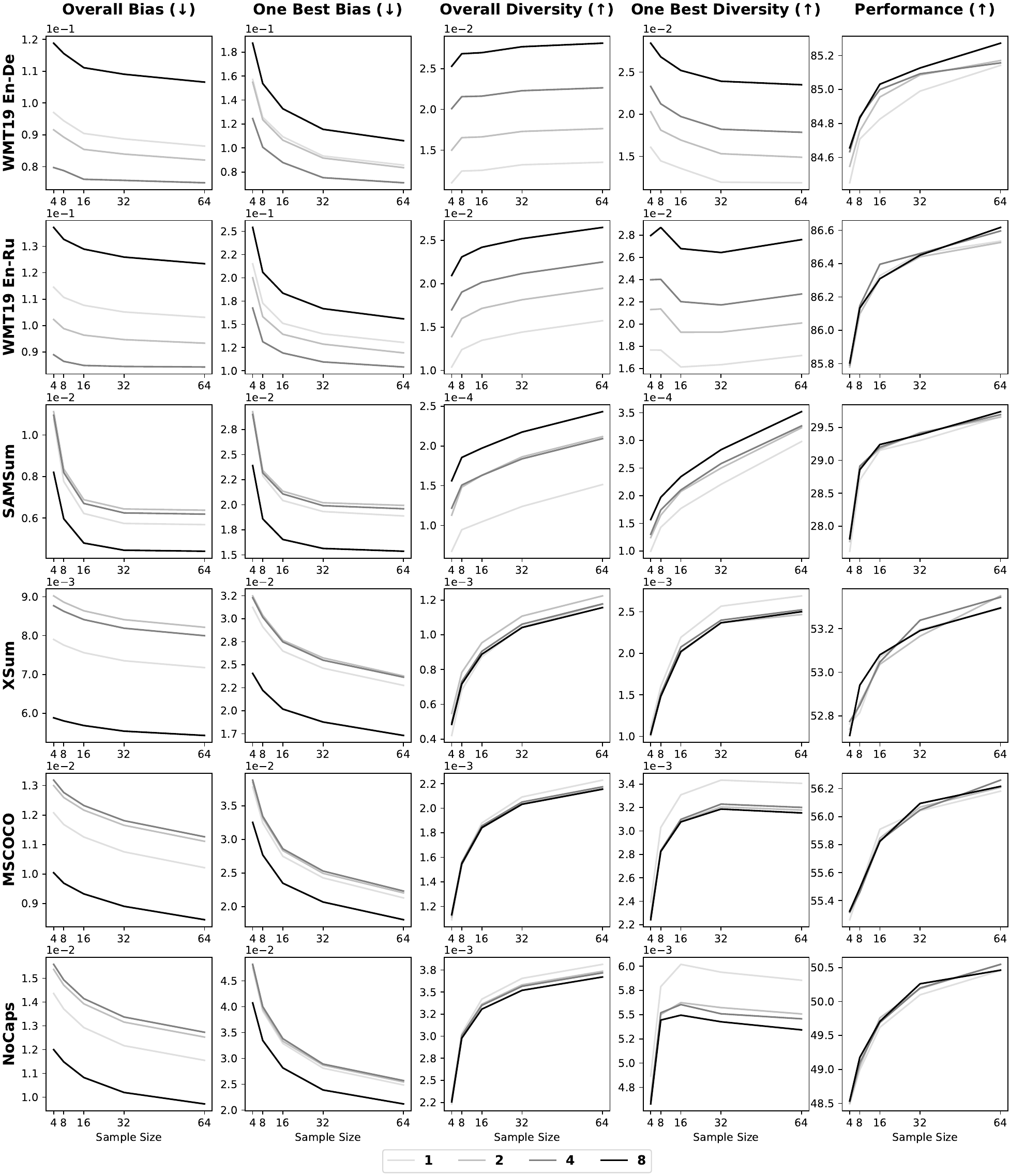}
    \caption{The relationship between bias, diversity, and performance on the first 1000 lines of each dataset in MBR decoding with pseudo-references generated by beam decoding. The notations are the same as Figure \ref{fig:mambr:bias_diversity:full:ancestral}.}
    \label{fig:bias_diversity:1000:mambr:beam}
\end{figure*}

\end{document}